%% file: neurips_2026.tex
\documentclass{article}

\PassOptionsToPackage{round}{natbib}

\usepackage[main, final]{neurips_2026}

\usepackage[utf8]{inputenc}                         
\usepackage[T1]{fontenc}                            
\usepackage[x11names]{xcolor}                     
\usepackage[pdfborder={0 0 0}, colorlinks=true, citecolor=DodgerBlue4, linkcolor=Blue4, urlcolor=blue]{hyperref}            
\usepackage{url}            
\usepackage{booktabs}       
\usepackage{amsfonts}       
\usepackage{nicefrac}       
\usepackage{microtype}      
\usepackage{subcaption}     

\usepackage{mathtools}
\usepackage{amssymb}
\usepackage{amsthm}
\usepackage{thmtools}
\usepackage{xspace}

\input{notation}

\input{theoremenvs}

\title{Kernel-based guarantees for nonlinear parametric models in Bayesian optimization}

\author{%
  Rafael Oliveira\\
  CSIRO Technology\\
  Sydney, Australia\\
  \texttt{rafael.dossantosdeoliveira@csiro.au} \\
}

\begin{document}

\maketitle

\begin{abstract}
    Modern Bayesian optimization and adaptive sampling methods increasingly rely on nonlinear parametric models, such as neural networks, yet theoretical guarantees for such models under adaptive data collection remain limited. Existing analyses largely focus on Gaussian processes, kernel machines, linear models, or linearized neural approximations, leaving a gap between theory and the nonlinear models used in practice. We develop a kernel-based framework for analyzing regularized parametric models trained on adaptively collected data. Our approach uses kernels over the parameter space to induce reproducing kernel Hilbert space structures over the corresponding model class, yielding confidence bounds for models trained with a broad class of regularized losses. We show how these bounds can establish convergence guarantees for approximate Thompson sampling and provide a foundation for analyzing broader likelihood-free Bayesian optimization methods. These results provide a unified route to analyzing nonlinear models in Bayesian optimization and related adaptive optimization settings.
\end{abstract}

\section{Introduction}
Bayesian optimization (BO) and adaptive experimental design provide principled frameworks for the optimization of expensive black-box functions by selecting data points sequentially in resource-constrained settings \citep{Rainforth2024}, which makes them attractive for problems where observations require costly simulations, physical experiments, or human evaluation \citep{Shahriari2016}.
The adaptive nature of the data collection process allows for choosing data points which are most informative for learning the model, typically improving data efficiency and reducing costs \citep{Greenhill2020, Rainforth2024}. However, this adaptiveness also introduces dependencies that invalidate typical assumptions of independent and identically distributed (\iid) data in classic statistical learning theory.
As a result, standard \iid learning-theoretic arguments do not directly apply, and performance guarantees typically require tools from bandits \citep{Srinivas2010, Chowdhury2017} and derivations based on martingales and self-normalized stochastic processes \citep{Pena2009article, Abbasi-Yadkori2012}.

Most existing regret analyses for BO rely on Gaussian processes \citep{Rasmussen2006}, kernelized models \citep{Chowdhury2017}, linear models \citep{Russo2014}, or linearized neural approximations \citep{Dai2022}. Although Gaussian processes and kernel methods can represent nonlinear functions, they admit linear representations in potentially infinite-dimensional feature spaces, which facilitate their analysis and enable tractable uncertainty estimates. However, these analyses do not directly address the nonlinear parameterizations of many models used in modern machine learning, including finite-width deep neural networks. This gap becomes especially relevant for likelihood-free Bayesian optimization (LFBO) methods \citep{Song2022a}, which learn acquisition functions or utility-based decision rules directly from data, often using flexible parametric models \citep{Tiao2021, Steinberg2024}, rather than deriving acquisitions from an explicit Gaussian process (GP) surrogate. Despite their practical appeal, regret guarantees for such nonlinear finite-dimensional acquisition models remain limited.

This paper develops a kernel-based framework for analyzing regularized nonlinear models trained on adaptively collected data. Our starting point is to view the parametric model class
\begin{equation*}
    \mdlspace=\{\model(\placeholder,\parameters):\parameters\in\paramspace\}
\end{equation*}
through a reproducing kernel Hilbert space (RKHS, \citealp{Scholkopf2001book}) induced by a positive-semidefinite kernel over the parameter space. This construction gives a Hilbert space geometry for nonlinear models, allowing regularization, functional gradients, and prediction errors to be controlled using RKHS tools. Moreover, under smoothness conditions, the induced model-dependent kernel can be dominated by standard kernels, such as Matérn kernels \citep[Ch. 4]{Rasmussen2006}, leading to GP-style predictive-variance bounds for parametric models used in LFBO, without requiring these models to be Gaussian processes or kernel machines.

We then apply this framework to randomized regularized policies. At each round, the model is trained with a regularization penalty around a fresh random initialization, and the next point is selected greedily by maximizing the trained model. This construction is closely related to randomize-then-optimize methods \citep{Bardsley2014, Dai2022}, and provides an approximate Thompson-sampling mechanism \citep{Russo2014} for nonlinear models. The random initialization is closely related to how models are deployed in practice and plays two key roles: it induces exploration, and, under sufficient regularization, it stabilizes the trained model around the initialized function.

Our main formal regret result focuses on the finite-domain approximate Thompson-sampling setting, where the randomized model is trained as a surrogate for the objective itself. We prove that, under realizability, a unique optimum, and prior regularity assumptions expressed in the RKHS-induced pseudometric, the algorithm obtains a sublinear cumulative regret. The proof combines three ingredients: a self-normalized generalization bound for the randomized regularized estimator, an effective small-ball condition ensuring that favorable initializations occur often enough, and a variance-decay argument showing that repeated visits to the optimum shrink a GP-style uncertainty term appearing in the regret decomposition.

Our results provide a first regret guarantee for randomized nonlinear parametric models in an LFBO-related setting, moving beyond analyses based on GP surrogates, kernel machines, or linearized neural models. Beyond this specific setting, the framework suggests a route for developing principled non-GP-based Bayesian optimization, active learning, and adaptive experimental design methods where flexible parametric models admit kernel-based approximation and regret analyses.

\section{Problem setting and background}
We consider an adaptive data collection setting, as it is the typical case in Bayesian optimization, bandits and general adaptive experimental design problems. In the following, we describe the basic problem formulation, state our assumptions on the training objectives for the models, and connect the formulation with existing frameworks.

\subsection{Adaptive data collection}
At each round $\iteridx \in \N$, an algorithm selects a point $\location_\iteridx \in \domain$ based on a model $\model_{\iteridx-1}: \domain\to\obsspace$ and then observes $\observation_\iteridx \in \obsspace$, updating a dataset $\dataset_\iteridx := \braces{(\location_\iteridx,\observation_\iteridx)} \cup \dataset_{\iteridx-1}$. We distinguish the raw observations $\observation_i$ from the training labels $\clabel_i$, which may be derived from the observations, e.g., through a utility function, or simply satisfy $\clabel_i=\observation_i$ in regression. We consider parametric models $\model: \domain\times\paramspace \to \obsspace$, where the parameters at round $\iteridx \geq 1$ are chosen by minimizing an empirical loss function with respect to $\model_\parameters := \model(\placeholder, \parameters)$:
\begin{equation}
    \Loss_\iteridx(\model_\parameters) = \regfun_\iteridx(\model_\parameters) + \sum_{i=1}^\iteridx \loss(\model_\parameters(\location_i), \clabel_i), \qquad \parameters\in \paramspace, \quad \iteridx\in\N,
    \label{eq:loss}
\end{equation}
where $\loss: \obsspace\times\obsspace\to\R$ defines a pointwise approximation loss with respect to the data, and $\regfun_\iteridx(\model_\parameters)$ is a regularization penalty, such as $\ell_2$ regularization $\regfun_\iteridx(\model_\parameters) = \regfactor_\iteridx \norm{\parameters}_2^2$, for a given regularization factor $\regfactor_\iteridx > 0$. Our analysis will focus on a compact domain $\domain$ and scalar-valued observations $\obsspace\subseteq\R$, though most of our results can be extended to the vector-valued setting after mild adjustments. 

\subsection{Loss functions}
We consider strictly convex smooth loss functions, which are common in the machine learning literature. Formally, we make the following assumption on the individual loss functions:

\begin{assumption}[Loss]
\label{a:loss}
    The loss function $\loss: \obsspace\times\obsspace\to\R$ is such that the map $\loss_\clabel := \anyscalar \mapsto \loss(\anyscalar, \clabel)$ is twice differentiable and strongly convex with $\ddot\loss_\clabel(\anyscalar) \geq \lossfactor > 0$, for all $\clabel, \anyscalar\in\obsspace$.
\end{assumption}

Most loss functions commonly found in the ML literature are strongly convex, at least when restricted to model outputs lying in a compact interval. A classic example includes the cross-entropy loss:
\begin{equation}
    \loss_{\mathrm{CE}}(\model(\location), \clabel) = -\clabel \log \model(\location) - (1-\clabel)\log(1-\model(\location)),
\end{equation}
whose second derivative is $\ddot\loss_\clabel(\anyscalar) = \frac{\clabel}{\anyscalar^2} + \frac{1-\clabel}{(1-\anyscalar)^2} > 0$, which is strongly convex whenever $\anyscalar$ remains within a compact subset of the open unit interval $(0,1)$ (e.g., when we model log-probabilities). Another example is the squared error loss:
\begin{equation}
    \loss_{\mathrm{SE}}(\model(\location),\clabel) = \frac{1}{2}(\model(\location) - \clabel)^2,
\end{equation}
which is strongly convex by construction, with $\lossfactor = 1$.

\subsection{Bayesian optimization}
The driving example framework we consider is Bayesian optimization. However, parallels can be drawn to multi-armed bandits and active learning. In the BO setting, we seek to estimate the global optimum of an objective function $\objective: \domain \to \R$,
\begin{equation}
    \location^\star \in \argmax_{\location\in\domain} \objective(\location).
    \label{eq:objective}
\end{equation}
We are given a maximum budget of function evaluations $\niter \geq 1$, selecting $\location_\iteridx \in \domain$ and observing $\observation_\iteridx = \objective(\location_\iteridx) + \obsnoise_\iteridx$, where $\obsnoise_\iteridx$ represents zero-mean observation noise, for $\iteridx \in \range{\niter}$. To derive their point selection strategy for $\location_\iteridx$, BO frameworks usually place a prior distribution over the objective function, such as a Gaussian process $\objective \sim \gp(0,\kernel)$, given a positive-semidefinite kernel $\kernel:\domain\times\domain\to\R$. Then decisions are made by maximizing an acquisition function $\af_\iteridx: \domain\to\R$,
\begin{equation}
    \location_\iteridx \in \argmax_{\location\in\domain} \af_{\iteridx-1}(\location), \quad \iteridx \in \range{\niter}.
\end{equation}
The acquisition function is generally formulated as an expected utility $\af_\iteridx(\location) = \expectation[\utility_\iteridx(\observation) \mid \location, \dataset_{\iteridx}]$, where $\utility_\iteridx:\obsspace\to\R$ is a non-negative utility function, such as improvement $\utility_\iteridx(\observation) := \max\braces{\observation - \threshold_\iteridx, 0}$ over a threshold $\threshold_\iteridx \in \R$, e.g., $\threshold_\iteridx := \max_{i\leq\iteridx}\observation_i$. The expectation is set to marginalize both epistemic and aleatoric uncertainty stemming, respectively, from the model and the observation noise. Alternatively, other popular acquisition functions are not directly expected utilities but work surprisingly well in practice, such as Thompson sampling, $\af_\iteridx = \hat\objective_\iteridx$, where $\hat\objective_\iteridx$ is a sample from the GP posterior given the available data \citep{Russo2014}, and the upper confidence bound $\af_\iteridx(\location) = \expectation[\objective(\location)\mid \dataset_\iteridx] + \beta_\iteridx \sqrt{\variance[\objective(\location)\mid\dataset_\iteridx]}$, where $\beta_\iteridx > 0$ is a confidence parameter, and $\variance[\placeholder]$ represents the variance. We refer the reader to \citet{Garnett2022book} for further background on BO.

Beyond the classic GP-based formulation, BO methods have been extended to use other types of models, such as Bayesian neural networks \citep{Snoek2012, Li2024bnnbo}, and likelihood-free Bayesian optimization (LFBO) methods \citep{Tiao2021, Song2022a, Steinberg2024, Oliveira2026generative}. The latter bypass the need for an explicit surrogate for the objective function $\objective$ and instead model the acquisition function $\af_\iteridx$ directly by fitting a model $\model_\iteridx \approx \af_\iteridx$ according to observation-based utilities $\dataset_\iteridx^\utility := \braces{(\location_i, \utility_\iteridx(\observation_i))}_{i=1}^\iteridx$. In this case, one uses utilities $\utility_\iteridx$ which can be directly computed from the data, e.g., $\utility_\iteridx(\observation) = \indicator[\observation \geq \threshold_\iteridx]$, where $\threshold_\iteridx$ is typically an empirical quantile of the marginal observations distribution \citep{Tiao2021}. We mainly focus on the LFBO setting, though our results are also applicable to methods that model the objective function using a parametric surrogate, such as neural Thompson sampling \citep{Zhang2021nts}.

\section{RKHS view of nonlinear models}
To analyze nonlinear parametric models under adaptive data settings, we need a function space in which distances, gradients, and confidence widths can be controlled. Hilbert spaces are particularly useful for this purpose because they provide an inner product geometry, allowing gradient-based arguments and self-normalized concentration inequalities \citep{Abbasi-Yadkori2012, Martinez-Taboada2026}. Since our observations involve model evaluations $\model(\location, \parameters)$, we also require point evaluation to be continuous. Hilbert spaces of functions satisfying this property are precisely reproducing kernel Hilbert spaces \citep[Def. 4.18]{Steinwart2008}, which motivates embedding the model class into an RKHS.

\paragraph{Model-induced RKHS.}
To analyze the learned model's approximation error, we embed the parametric model class:
\begin{equation}
    \mdlspace := \braces{\model(\cdot, \parameters) \mid \parameters \in \paramspace}.
    \label{eq:model-space}
\end{equation}
into an RKHS. Assume that $\model(\location,\placeholder) \in \Hspace_\paramspace$, where $\Hspace_\paramspace$ is an RKHS of real-valued functions over $\paramspace$. For instance, if the model is simply linear with respect to the parameters, i.e., $\model(\location,\parameters) = \parameters \cdot \vec\feature(\location)$, for some fixed feature map $\feature:\domain\to\paramspace$, then $\model(\location,\cdot)$ lies in the RKHS of the linear kernel $\kernel_\paramspace(\parameters, \parameters') = \parameters\cdot\parameters'$, for $\parameters,\parameters'\in\paramspace$. This is the case for linear regression and neural networks in the infinite-width limit \citep{Jacot2018}. More generally, if $\model(\location, \parameters)$ is at least $\anyscalar$ times differentiable with respect to $\parameters\in\paramspace$, for $\anyscalar\in\N$ such that $\anyscalar \geq \nu + \frac{\paramdim}{2}$, with square-integrable derivatives over $\paramspace\subseteq\R^\paramdim$, then $\model(\location, \placeholder)$ can be shown to lie in the RKHS of a Mat\'ern kernel $\kernel_\paramspace$ with smoothness parameter $\nu > 0$ \citep{Rasmussen2006} under mild assumptions on the domain and the parameter space. Having a kernel $\kernel_\paramspace$ and the associated RKHS $\Hspace_\paramspace$ of functions over $\paramspace$, we can use the latter as a feature space to construct an RKHS over $\domain$ which contains $\mdlspace$. An application of a classic RKHS result \citet[see][Thm. 4.21]{Steinwart2008} then yields the following construction.

\begin{lemma}
    \label{thr:model-rkhs}
    Assume that $\model(\location;\cdot) \in \Hspace_\paramspace$, for all
    $\location\in\domain$, where $\Hspace_\paramspace$ is a reproducing kernel Hilbert space associated with a positive-definite kernel $\kernel_\paramspace:\paramspace\times\paramspace\to\R$. It then follows that:
    \begin{equation*}
        \Hspace_\mdlspace = \{\anyfunction: \domain \to \R \mid \exists \weight\in\Hspace_\paramspace: \anyfunction(\location) = \inner{\weight, \model(\location, \cdot)}_{\Hspace_\paramspace}, \forall \location\in\domain \}
    \end{equation*}
    constitutes the only RKHS for which $\kernel_\mdlspace: (\location, \location') \mapsto \inner{\model(\location, \cdot), \model(\location', \cdot)}_{\Hspace_\paramspace}$ is the reproducing kernel and is equipped with the norm:
    \begin{equation*}
        \norm{\anyfunction}_{\mdlspace} := \inf\{\norm{\weight}_{\Hspace_\paramspace} : \weight \in \Hspace_\paramspace, \anyfunction(\location) = \inner{\weight, \model(\location, \cdot)}_{\Hspace_\paramspace}, \forall\location\in\domain \}.
        \label{eq:rkhs-norm}
    \end{equation*}
\end{lemma}

\paragraph{RKHS superspace.} Bounds expressed as a function of $\kernel_\mdlspace$ can be difficult to estimate in practice due to the intrinsic dependence of $\kernel_\mdlspace$ on the model architecture. However, it is often possible to find a better known kernel $\kernel$ which provides suitable upper bounds for quantities dependent on $\kernel_\mdlspace$. One practical instance is in the case of sufficiently smooth models with functions that lie in Sobolev spaces, which coincide with the RKHS of Mat\'ern kernels \citep{Kanagawa2025}. This is, for example, the case for neural networks with smooth activations \citep{Vakili2023} and generally models with bounded derivatives, as in the next result. Such RKHSs have well known spectral characterizations, which aid in deriving performance guarantees \citep{Srinivas2010}. Our proofs are located in \autoref{sec:proofs}.

\begin{lemma}
    \label{thr:model-matern}
    Let $\domain\subset\R^\locdim$ be compact, and let $\kernel:\domain\times\domain\to\R$ be a Mat\'ern kernel with smoothness parameter $\nu > 0$ with RKHS $\Hspace_\kernel$. Assume that, for an integer $\anyscalar \geq \nu + \frac{\locdim}{2}$,
    \begin{equation*}
        \sup_{\location\in\domain} \norm{D_\location^{\vec\alpha} \model(\location, \placeholder)}_{\Hspace_\paramspace} < \infty, \qquad \forall\vec\alpha \in \N_0^\paramdim: \norm{\vec\alpha}_1 \leq \anyscalar,
    \end{equation*}
    where $\norm{\vec\alpha}_1 = \sum_{i=1}^\locdim \alpha_i$, for $\vec\alpha = [\alpha_i]_{i=1}^\locdim \in \N_0^\locdim$, and $D_\location^{\vec\alpha} \anyfunction := \frac{\partial^{\norm{\vec\alpha}_1} \anyfunction}{\partial\slocation_1^{\alpha_1}\dots \partial\slocation_\locdim^{\alpha_\locdim}}$ denotes the multi-index partial derivative.\footnote{For example, $D_\location^{\vec\alpha} \anyfunction(\location) = \frac{\partial^3 \anyfunction(\location)}{\partial\slocation_1^2 \partial\slocation_3}$, for $\alpha = (2, 0, 1)$, and $\location := (\slocation_1, \slocation_2, \slocation_3) \in \R^3$.} Then, there exists a finite constant $\bound_\mdlspace > 0$, such that $\kernel_\mdlspace \preceq \bound_\mdlspace^2 \kernel$, and $\Hspace_\mdlspace \subset \Hspace_\kernel$.
\end{lemma}

\paragraph{Examples.}
Consider two illustrative examples, a linear and a nonlinear parametric model. For a \emph{linear-in-parameters} model $\model(x,\parameter)=\inner{\parameter,\feature(x)}$ with the linear parameter kernel $\kernel_{\paramspace}(\parameter,\parameter')=\inner{\parameter,\parameter'}$, we have $\norm{\partial_x^{\vec\alpha}\model(x,\cdot)}_{\Hspace_{\paramspace}}=\norm{\partial_x^{\vec\alpha}\feature(x)}_2$. Thus, the conditions of \autoref{thr:model-matern} reduce to boundedness of the corresponding input derivatives of the feature map. For a \emph{nonlinear} example, consider a finite-width neural network $\model$ with smooth activations, defined on compact, sufficiently regular domains $\domain\subset\R^d$ and $\paramspace\subset\R^{\paramdim}$. Choosing a Matérn parameter kernel $\kernel_{\paramspace}$ yields $\Hspace_{\paramspace}\simeq H^r(\paramspace)$, where $H^r(\paramspace)$ is the Sobolev space of order $r=\nu_{\paramspace}+\paramdim/2$. For integer $r$, this space consists of functions with square-integrable weak derivatives up to order $r$, so $\norm{\partial_x^{\vec\alpha}\model(x,\cdot)}_{\Hspace_{\paramspace}}$ is controlled by the mixed derivatives $\partial_\parameter^\beta\partial_x^{\vec\alpha}\model(x,\parameter)$, $|\beta|\leq r$. In particular, neural networks with smooth ($C^\infty$) activations, such as $\tanh$ or sigmoid, have uniformly bounded mixed derivatives of any fixed finite order on $\domain\times\paramspace$, satisfying the lemma's conditions. Choosing a Matérn input kernel $\kernel$ with $\Hspace_{\kernel}\simeq H^s(\domain)$, where $H^s(\domain)$ is the corresponding Sobolev space of order $s$, then gives the continuous embedding $\Hspace_{\mdlspace}\hookrightarrow\Hspace_{\kernel}$ and hence $\kernel_{\mdlspace}\preceq\bound_{\mdlspace}^2\kernel$ for some finite $\bound_{\mdlspace}$.

\paragraph{Regularization in the induced RKHS.}
The RKHS perspective allows us to formulate model-space regularization via RKHS norms. Namely, as $\model(\location, \placeholder)\in\Hspace_\paramspace$, by the reproducing property of $\kernel_\paramspace$, we have that $\model(\location, \parameters) = \inner{\kernel_\paramspace(\placeholder, \parameters), \model(\location, \placeholder)}_{\Hspace_\paramspace}$. Therefore, \autoref{thr:model-rkhs} implies that $\norm{\model(\pch, \parameters)}_\mdlspace^2 \leq \kernel_\paramspace(\parameters, \parameters)$. This fact allows us to define regularized norms over the RKHS $\Hspace_\mdlspace$ which are computable as a function of the model parameters. Regularized losses of the form:
\begin{equation}
    \Loss_\iteridx(\model) = \regfactor\norm{\model}_\mdlspace^2 + \sum_{i=1}^\iteridx \loss(\model(\location_i), \clabel_i), \qquad \iteridx\geq 1, \quad \model \in \Hspace_\mdlspace,
\end{equation}
are strongly convex in $\Hspace_\mdlspace$, given any $\regfactor>0$, under our assumptions on the pointwise loss (see \autoref{a:loss}), as $\nabla_\model^2 \loss(\model(\location), \clabel) = \ddot\loss_\clabel(\model(\location)) \feature(\location) \otimes \feature(\location) \succeq 0$, where $\feature(\location) = \kernel_\mdlspace(\cdot,\location)$ is the canonical feature map.\footnote{We use $\anyoperator \succeq 0$ to denote that $\anyoperator$ is positive semidefinite, i.e., $\inner{\anyfunction, \anyoperator\anyfunction}_\Hspace \geq 0$, for all $\anyfunction\in\Hspace$, given a bounded linear operator $\anyoperator$ on a Hilbert space $\Hspace$.} Strong convexity allows for controlling approximation errors between the loss minimizer and other elements, which becomes a key ingredient for generalization error bounds \citep{Karimi2016, Oliveira2026generative}.

\section{On the cumulative regret of likelihood-free Bayesian optimization}
In this section, we aim to analyze the optimization performance of LFBO algorithms \citep{Song2022a}. In their general setting, an algorithm learns an approximation to the acquisition function $\model_\iteridx \approx \af_\iteridx$ by minimizing a loss function over utility data $\dataset_\iteridx^\utility := \braces{(\location_i, \mlabel_{\iteridx,i})}_{i=1}^\iteridx$, where $\mlabel_{\iteridx,i} = \utility_\iteridx(\observation_i)$, for $i\leq\iteridx$ and $\iteridx\in\N$. The next data point is selected by optimizing over the learned model:
\begin{equation}
    \location_{\iteridx} \in \argmax_{\location\in\domain} \model_{\iteridx-1}(\location), \qquad \iteridx \geq 1,
\end{equation}
where $\model_\iteridx := \model(\placeholder, \parameters_\iteridx)$, and $\parameters_\iteridx$ represents the empirical loss minimizer \eqref{eq:loss}. An observation is collected $\observation_\iteridx = \objective(\location_\iteridx) + \obsnoise_\iteridx$, where $\obsnoise_\iteridx$ is observation noise, and the process repeats for up to $\niter \geq 1$ iterations.

The goal of the algorithm is to estimate the global optimum of $\objective$ \eqref{eq:objective}. Optimization performance can be characterized by analyzing the algorithm's regret. The instant regret of an algorithm quantifies how much is lost in performance for selecting $\location_\iteridx\in\domain$ when the global optimum is at $\location^\star\in\domain$:
\begin{equation}
    \regret_\iteridx := \objective(\location^\star) - \objective(\location_\iteridx), \qquad \iteridx\in\N.
\end{equation}
The cumulative regret $\Regret_\niter := \sum_{\iteridx\leq\niter}\regret_\iteridx$ is typically the target performance metric, as it captures how often the algorithm makes sub-optimal choices. We also have an upper bound on the simple regret with it as $\min_{\iteridx\leq\niter} \regret_\iteridx \leq \frac{\Regret_\niter}{\niter}$. Therefore, if $\Regret_\niter$ grows at most sub-linearly, $\regret_\iteridx$ eventually vanishes.

\subsection{Randomized policies}
A deterministic acquisition learner does not, by itself, provide exploration. Once the learned model is optimized greedily, any uncertainty not encoded in the trained acquisition target or in the training procedure is lost. For example, deterministic BORE-type classifiers can fail to explore beyond the region of currently positive training labels \citep{Oliveira2022}. Hence, introducing a certain amount of randomness in the training of the model is beneficial. We, therefore, consider randomized models:
\begin{equation}
    \parameters_\iteridx \in \argmin_{\parameters\in\paramspace} \frac{\regfactor_\iteridx}{2}\norm{\model_\parameters - \model_{\iteridx,0}}_\mdlspace^2 + \sum_{i=1}^\iteridx \loss(\model_\parameters(\location_i), \clabel_i), \qquad \parameters_{\iteridx, 0} \sim \paramprior, \quad \iteridx\geq 1,
    \label{eq:random-loss}
\end{equation}
for a non-decreasing, real-valued, positive sequence $\seqbracks{\regfactor_\iteridx}_{\iteridx\in\N}$. We then set $\model_\iteridx := \model(\placeholder, \parameters_\iteridx)$, for $\iteridx\geq 1$, with $\model_0 := \model_{0,0}$. Thus, by \autoref{thr:model-rkhs}, we are penalizing large deviations from the random initialization, according to the pseudometric defined by the RKHS norm in $\paramspace$:
\begin{equation}
    \paramdist(\parameters,\parameters') := \norm{\model_\parameters - \model_{\parameters'}}_\mdlspace \leq \norm{\kernel_\paramspace(\pch, \parameters) - \kernel_\paramspace(\pch, \parameters')}_{\Hspace_\paramspace} = \sqrt{\kernel_\paramspace(\parameters, \parameters) - 2\kernel_\paramspace(\parameters, \parameters') + \kernel_\paramspace(\parameters', \parameters')}.
\end{equation}
The next data point $\location_{\iteridx+1}$ is located at the maximizer of the randomized model, for $\iteridx\in\braces{0,\dots,\niter-1}$. When the amount of data is small, the strength of the regularization then prevents the model from overfitting, staying close to the random initialization $\model_{\iteridx,0}(\location)$ at points $\location$ away from the data. As the algorithm progresses, this promotes exploration of the design space.

We highlight that most modern machine learning frameworks implement their main model building blocks with randomly initialized parameters, e.g., PyTorch \citep{Paszke2019pytorch}, besides providing options for $\ell_2$-based regularization with their optimizers, such as weight decay. In addition, as no model is trained for an infinite amount of time, early stopping provides an implicit level of regularization that prevents models from going too far from their initialization \citep{Fleming1990, Jacot2018}. However, we make explicit the dependence on the initialization here for our analysis, which is essential to ensure sufficient exploration. This construction also agrees with recent results on explicit regularization for neural networks under noisy data \citep{Calvo2025}.

\begin{assumption}[Prior]
    \label{a:prior}
    There exist constants $\anyconstant_\star, \pradius_0, \deff \in (0,\infty)$ such that, for any $\parameters_0 \sim \paramprior$,
    \begin{equation*}
        \prob{\paramdist(\parameters_\star, \parameters_0) \leq \pradius} \geq \anyconstant_\star\pradius^\deff, \quad \forall \pradius \in (0, \pradius_0),
    \end{equation*}
    and there exists a non-decreasing function $\pfunction: (0,\infty) \to (0,\infty)$ such that, for all $\anyscalar > 0$,
    \begin{equation*}
        \prob{\paramdist(\parameters_\star, \parameters_0) > \pfunction(\anyscalar)} \leq e^{-\anyscalar}.
    \end{equation*}
\end{assumption}

In practice, the lower-tail condition is a local support requirement in the geometry induced by the model class, while the upper-tail condition controls unusually poor initializations. For example, for $\paramspace\subseteq\R^\paramdim$, the map $\parameters\mapsto \model(\placeholder,\parameters)$ is locally Lipschitz from the Euclidean distance to $\paramdist$, and $\paramprior$ has a density bounded below in a neighborhood of $\parameters_\star$, then the small-ball condition holds with $\deff\le \paramdim$. If, in addition, $\paramdist(\parameters,\parameters_\star)$ has at most linear growth in $\|\parameters-\parameters_\star\|_2$ and $\paramprior$ has sub-Gaussian Euclidean tails, then the upper-tail condition holds with $\pfunction(\anyscalar)$ as $O(1+\sqrt\anyscalar)$. These assumptions, therefore, include standard finite-dimensional Gaussian initializations, up to the constants and effective dimension induced by the RKHS pseudometric.

\subsection{Generalization error bounds}
\label{sec:error-bounds}
For our theoretical analysis, we need generalization error bounds for the model's predictions. The main advantage of the RKHS-based framework is that it allows us to quantify error bounds with respect to the RKHS minimizer of the regularized loss:
\begin{equation}
    \hat\model_\iteridx \in \argmin_{\model \in \Hspace_\mdlspace} \Loss_\iteridx(\model), \qquad \iteridx\geq 1,
\end{equation}
which provides several advantages when compared to trying to bound approximation errors directly with respect to $\model_\iteridx := \model(\placeholder,\parameters_\iteridx)$. For instance, $\model_\iteridx$ is basically an infinite-dimensional linear model in $\Hspace_\mdlspace$. The pointwise loss function's gradient with respect to $\anyfunction \in \Hspace_\mdlspace$ is simply $\nabla_\anyfunction \loss(\anyfunction(\location), \clabel) = \dot\loss_\clabel(\anyfunction(\location)) \nabla_\anyfunction \anyfunction(\location) = \dot\loss_\clabel(\anyfunction(\location)) \feature(\location)$, where $\feature(\location) := \kernel_\mdlspace(\placeholder, \location)$ is the canonical feature map associated with $\kernel_\mdlspace$, for $\location\in\domain$ and $\clabel\in\obsspace$. Furthermore, by \autoref{a:loss}, $\loss(\anyfunction(\location), \clabel)$ is also strictly convex with respect to $\anyfunction \in \Hspace_\mdlspace$, as $\nabla_\anyfunction^2 \loss(\anyfunction(\location), \clabel) = \ddot\loss_\clabel(\anyfunction(\location)) \feature(\location) \otimes \feature(\location) \succeq 0$.\footnote{We use $\anyoperator \succeq 0$ to denote that $\anyoperator$ is positive semidefinite, i.e., $\inner{\anyfunction, \anyoperator\anyfunction}_\Hspace \geq 0$, for all $\anyfunction\in\Hspace$, given a bounded linear operator $\anyoperator$ on a Hilbert space $\Hspace$.} Considering these observations, we are able to derive the following result.

\begin{theorem}
    \label{thr:error-bound}
    Let $\braces{\filtration_\iteridx}_{\iteridx=0}^\infty$ be a filtration such that $\location_\iteridx$ is $\filtration_{\iteridx-1}$-measurable and $\clabel_\iteridx$ is $\filtration_\iteridx$-measurable for all $\iteridx\geq 1$. Let $\parameters_\iteridx$ denote the minimizer of the loss in \autoref{eq:random-loss}. Assume that there exists $\parameters_\star \in \paramspace$ such that $\dot\loss_{\clabel_\iteridx}(\model(\location_\iteridx, \parameters_\star))$ is conditionally $\sigma_\loss^2$-sub-Gaussian given $\filtration_{\iteridx-1}$, for all $\iteridx\geq 1$, and let $\loss$ satisfy \autoref{a:loss}. Let $\regfactor_0>0$ be such that $\regfactor_0 \leq \inf_{\iteridx\in\N}\regfactor_\iteridx$.
    Then, given any $\delta \in (0,1]$,
    \begin{equation*}
        \prob{\forall \iteridx \geq 0, \quad \abs{\model(\location, \parameters_\iteridx) - \model(\location, \parameters_\star)} \leq 2\widehat\beta_\iteridx(\delta) \norm{\feature(\location)}_{\Hessian_\iteridx^{-1}}, \quad \forall \location \in \domain} \geq 1 -\delta,
    \end{equation*}
    where
    $\widehat\beta_\iteridx(\delta) := \sqrt{\regfactor_\iteridx}\norm{\model_\star - \model_{\iteridx,0}}_\mdlspace + \sigma_\loss\sqrt{2\lossfactor^{-1}\log(\det(\eye + \lossfactor\regfactor_0^{-1}\mat\kernel_\iteridx^\mdlspace)^\half/\delta)}$,
    $\mat\kernel_\iteridx^\mdlspace := [\kernel_\mdlspace(\location_i, \location_j)]_{i,j=1}^\iteridx \in \R^{\iteridx\times\iteridx}$,
    and $\Hessian_\iteridx := \regfactor_\iteridx \idop + \lossfactor \sum_{i=1}^\iteridx \feature(\location_i)\otimes\feature(\location_i)$.
\end{theorem}

The proof is a pointwise-specialized version of the RKHS loss-geometry and self-normalized martingale argument of \citet{Oliveira2026generative}, with the regularization term centered at the random initialization.
The main implication is that we can control the approximation error with respect to a true underlying model $\model_\star := \model(\placeholder,\parameters_\star)$ controlling the data-generation process. All our proofs are found in \autoref{sec:proofs}.

The error bound in \autoref{thr:error-bound} demands a more thorough characterization of the model-dependent kernel $\kernel_\mdlspace$, which appears in the main quantities controlling the error. However, the next result allows us to replace $\kernel_\mdlspace$ in the error bound with any better known model-agnostic dominant kernel.

\begin{corollary}
    \label{thr:error-gp}
    Assume there exists a positive-semidefinite kernel $\kernel: \domain\times\domain\to\R$ such that $\kernel - \kernel_\mdlspace$ defines a positive-semidefinite kernel over $\domain$, i.e., $\kernel \succeq \kernel_\mdlspace$. Then, under the settings in \autoref{thr:error-bound}, the following holds with probability at least $1-\delta$:
    \begin{equation*}
        \forall \iteridx \geq 0, \quad \abs{\model(\location, \parameters_\iteridx) - \model(\location, \parameters_\star)} \leq 2\beta_\iteridx(\delta) \sigma_\iteridx(\location), \quad \forall \location \in \domain,
    \end{equation*}
    where
    $\beta_\iteridx(\delta)
    = \sqrt{\frac{\regfactor_\iteridx}{\regfactor_0}}\norm{\model_\star - \model_{\iteridx,0}}_\mdlspace
    + \sqrt{\frac{2\sigma_\loss^2}{\lossfactor\regfactor_0}\log(\det(\eye + \lossfactor\regfactor_0^{-1}\mat\kernel_\iteridx)^\half/\delta)}$, for a given $\regfactor_0 \leq \inf_{\iteridx\in\N} \regfactor_\iteridx$,
    and $\sigma_\iteridx^2$ is the posterior predictive variance of a zero-mean GP with covariance function $\kernel$. Namely,
    \begin{equation*}
        \sigma_\iteridx^2(\location) := \kernel(\location,\location) - \vec\kernel_\iteridx(\location)^\transpose(\mat\kernel_\iteridx + {\regfactor_0}{\lossfactor}^{-1}\eye)^{-1}\vec\kernel_\iteridx(\location),
    \end{equation*}
    where $\vec\kernel_\iteridx(\location) := [\kernel(\location,\location_i)]_{i=1}^\iteridx \in \R^\iteridx$, and $\mat\kernel_\iteridx := [\kernel(\location_i, \location_j)]_{i,j=1}^\iteridx \in \R^{\iteridx\times\iteridx}$.
\end{corollary}

Although the learned model need not be a GP or kernel machine, \autoref{thr:error-gp} shows that its pointwise approximation error can be bounded by standard GP quantities.  This reduction allows standard GP information-gain bounds for $\kernel$ to be used even when the model being trained is a nonlinear parametric model rather than a Gaussian process. The following result shows that conditions in \autoref{thr:error-gp} can be satisfied for sufficiently smooth models, such as neural networks with smooth activations.

\begin{remark}
The initialization-dependent term in \(\beta_t(\delta)\) can be bounded uniformly in two ways. First, if the parameter-space kernel satisfies \(\sup_{\theta\in\Theta}k_\Theta(\theta,\theta)\leq\kappa^2<\infty\), then \(\rho(\parameters_\star,\theta_{t,0})\leq2\kappa\) for every initialization, so the conclusion of \autoref{thr:error-gp} holds with \(\norm{\model_\star - \model_{\iteridx,0}}_\mdlspace\) replaced by \(2\kappa\), without changing the confidence level. Alternatively, under \autoref{a:prior}, let \((\delta_t)_{t\geq1}\) be any positive sequence satisfying \(\sum_{t\geq1}\delta_t\leq\delta/2\). A union bound then yields a simultaneous confidence bound with probability at least \(1-\delta\), replacing \(\beta_t(\delta)\) by \(\sqrt{\lambda_\iteridx/\lambda_0}\psi_0(\log(1/\delta_\iteridx))+B_\iteridx(\delta/2)\), where \(B_\iteridx(\delta)\) denotes the self-normalized concentration term in \(\beta_t(\delta)\).
\end{remark}

\begin{remark}
    \label{rem:beta-log-t}
    For a finite domain $\card{\domain} < \infty$, the kernel matrix $\mat\kernel_\iteridx$ in \autoref{thr:error-gp} can have at most $\nfeatures = \card{\domain}$ nonzero eigenvalues, $\keigval_{\iteridx,1} \geq \dots \geq \keigval_{\iteridx,\nfeatures}$, and each is bounded by $\keigval_{\iteridx,1} \leq \tr (\mat\kernel_\iteridx) \in \bigo(\iteridx)$. Thus, $\log\det(\eye + \lossfactor\regfactor_0^{-1}\mat\kernel_\iteridx) \leq \sum_{i=1}^\nfeatures \log(1 + \lossfactor\regfactor_0^{-1}\keigval_{\iteridx,i}) \leq \nfeatures \log(1 + \lossfactor\regfactor_0^{-1}\keigval_{\iteridx,1})$. Therefore, we have that $\beta_\iteridx \in \bigo_p\paren*{\sqrt{\regfactor_\iteridx} + \sqrt{\card{\domain}\log \iteridx}}$ in the finite-domain case.\footnote{We use $\bigo_p$ to denote rates of asymptotic convergence in probability.}
\end{remark}

For improvement-based LFBO, recomputing thresholds can relabel past data and break the predictability assumptions used in the concentration argument. We expect this effect to be benign for the functioning of the algorithm, but a full regret analysis would require a separate stability argument. We therefore leave the complete treatment of time-varying LFBO utilities to future work. In the formal regret result below, we focus on the regression case $\mlabel_\iteridx:=\observation_\iteridx$, so that $\model_\iteridx\approx\objective$ and the randomized greedy policy recovers approximate Thompson sampling \citep{Zhang2021nts, Paria2022}. This setting isolates the main contribution of the analysis: controlling regret for finite-dimensional nonlinear randomized models under adaptive sampling. In addition, we keep our assumptions about the loss function generic (\autoref{a:loss}), not forcing it to be any specific regression loss.

\subsection{Approximate Thompson sampling}
\label{sec:ts-analysis}

For approximate Thompson sampling, the model $\model_\iteridx$ approximates the objective $\objective$ based on information available at time $\iteridx\geq 1$. The algorithm's regret at time $\iteridx\geq 1$ can be decomposed as:
\begin{equation}
    \begin{split}
        \regret_\iteridx &= \objective(\location^\star) - \model_{\iteridx-1}(\location^\star) + \model_{\iteridx-1}(\location^\star) - \model_{\iteridx-1}(\location_\iteridx) + \model_{\iteridx-1}(\location_\iteridx) - \objective(\location_\iteridx)\\
        &\leq \objective(\location^\star) - \model_{\iteridx-1}(\location^\star) + \model_{\iteridx-1}(\location_\iteridx) - \objective(\location_\iteridx),
    \end{split}
\end{equation}
since $\model_{\iteridx-1}(\location^\star) - \model_{\iteridx-1}(\location_\iteridx) \leq 0$, as $\location_\iteridx$ maximizes $\model_{\iteridx-1}$. In exact Bayesian Thompson sampling, the difference $\objective(x^\star)-\model_{\iteridx-1}(\location_\iteridx)$ vanishes in conditional expectation, since the objective and the sampled model have the same posterior distribution. Thus, bounding the expected regret requires only controlling the approximation error at the selected point $x_t$. In contrast, our analysis considers a fixed, unknown objective $\objective$ and does not rely on this Bayesian identity. We must therefore control the approximation errors at both $x^\star$ and $x_t$. Applying \autoref{thr:error-gp}, we obtain:
\begin{equation}
    \forall\iteridx\geq 1, \qquad \regret_\iteridx \leq 2\beta_{\iteridx-1}(\delta) (\sigma_{\iteridx-1}(\location^\star) + \sigma_{\iteridx-1}(\location_\iteridx)),
\end{equation}
which holds with probability at least $1-\delta$ (uniformly over time), for a given $\delta \in (0,1]$. The second term on the right-hand side is the term usually found in regret bounds for exact Thompson sampling \citep{Takeno2020} and GP-UCB \citep{Srinivas2010}, and it can be bound as a function of a GP's maximum information gain \citep{Vakili2021}. Therefore, our analysis focuses on the first term $\beta_\iteridx\sigma_\iteridx(\location^\star)$, which can be guaranteed to vanish if the true optimum $\location^\star$ is visited sufficiently often.

\begin{lemma}
    \label{thr:variance-decay}
    Assume a finite domain $\abs{\domain} < \infty$ and that $\objective = \model(\placeholder,\parameters_\star)$ with a unique global optimum $\braces{\location^\star} = \argmax_{\location\in\domain} \objective(\location)$.
    Under the assumptions in \autoref{thr:error-gp}, let $\regfactor_\iteridx$ be set such that%
    \footnote{Here, $\omega(\anyscalar_\iteridx)$ denotes a strict asymptotic lower bound, i.e., $\regfactor_\iteridx \in \omega(\log\iteridx) \iff \lim_{\iteridx\to\infty} \frac{\log\iteridx}{\regfactor_\iteridx} = 0$.}
    $\regfactor_\iteridx \in \omega(\log\iteridx)$ and $\sum_{\iteridx=1}^\infty \regfactor_\iteridx^{-\deff/2} = \infty$,
    and let $\paramprior$ satisfy the conditions in \autoref{a:prior}. Then,
    \begin{equation*}
        \sigma_\iteridx^2(\location^\star) \in \bigo\paren*{\paren*{\sum_{i=1}^\iteridx\regfactor_i^{-\deff/2}}^{-1}},
    \end{equation*}
    with probability at least $1-\delta$. In particular, if $\regfactor_\iteridx \asymp \log^\regfexp \iteridx$, for $\regfexp > 1$, with the same probability,
    \begin{equation*}
        \sigma_\iteridx^2(\location^\star) \in \bigo\paren*{\frac{(\log \iteridx)^{\regfexp\frac{\deff}{2}}}{\iteridx}}.
    \end{equation*}
\end{lemma}

\begin{theorem}[Approximate TS regret]
    \label{thr:ts-regret}
    Assume $\domain$ is finite and that the assumptions in \autoref{thr:variance-decay} hold. Let $\regfactor_\iteridx \asymp \log^q \iteridx$, for some $q > 1$. In addition, assume that $\pfunction(\anyscalar)$ in \autoref{a:prior} is $\bigo(\anyscalar^\ptailexp)$, for some $\ptailexp > 0$. Then, the cumulative regret for approximate Thompson sampling with random initialization is such that, with probability at least $1-\delta$,
    \begin{equation*}
        \Regret_\niter \in \bigo\paren*{
            \sqrt{\niter}(\log\niter)^{\ptailexp + \regfexp (1 + \nicefrac{\deff}{4})}
        }.
    \end{equation*}
\end{theorem}

\subsection{General acquisition functions}
The regret bound above is stated for approximate Thompson sampling, where $\model_\iteridx$ is trained as a randomized surrogate for $\objective$. This is the simplest case because $\location^\star$ is also the maximizer of the target function being learned, so the finite-domain gap condition can be applied directly. For more general LFBO methods, $\model_\iteridx$ instead approximates an acquisition target $\af_\iteridx$, constructed from utilities such as improvement or thresholded labels. The same analysis can be reused, provided the acquisition targets are eventually aligned with the objective in the following sense: there exists a limiting acquisition function $\af_\infty$ with
\begin{equation}
\argmax_{\location\in\domain} \af_\infty(\location)
\subseteq
\argmax_{\location\in\domain} \objective(\location),
\end{equation}
and, in the finite-domain case, $\af_\iteridx$ converges to $\af_\infty$ with a positive limiting gap at $\location^\star$. Under this condition, the proof of \autoref{thr:variance-decay} applies with $\af_\iteridx$ in place of $\objective$. Favorable initializations make the trained acquisition model sufficiently close to the target acquisition, the greedy maximizer is therefore preserved, and the lower-tail condition in \autoref{a:prior} gives sufficiently many visits to $\location^\star$.

The approximate Thompson sampling result should therefore be viewed as a concrete application of a broader theoretical framework for nonlinear acquisition models, rather than as a proposal for a new BO algorithm. In particular, the generalization bounds developed here provide tools for analyzing existing LFBO methods whose acquisition functions are learned using regularized nonlinear models. Extending the regret guarantee to general LFBO acquisition functions additionally requires acquisition-specific alignment arguments and, for time-varying utilities, controlling the effects of changing thresholds and relabelling past observations, which we reserve for future work.

\section{Practical implementation and empirical illustrations}
\label{sec:experiments}

\paragraph{Kernel-based regularization.}
The RKHS distance $\paramdist$ in \eqref{eq:random-loss} is generally difficult to compute for nonlinear parametric models. In practice, we instead use the computable upper bound $d_\paramspace(\parameters,\parameters') := \norm{\kernel_\paramspace(\pch,\parameters)-\kernel_\paramspace(\pch,\parameters')}_{\Hspace_\paramspace}$, whose square is $\kernel_\paramspace(\parameters,\parameters)-2\kernel_\paramspace(\parameters,\parameters')+\kernel_\paramspace(\parameters',\parameters')$. The regularizer in \eqref{eq:random-loss} is then replaced by $\regfactor_\iteridx d_\paramspace^2(\parameters,\parameters_{\iteridx,0})/2$. For translation-invariant kernels, like the squared-exponential and Mat\'ern families, this further simplifies to $-\regfactor_\iteridx\kernel_\paramspace(\parameters,\parameters_{\iteridx,0})$, since $\kernel_\paramspace(\parameters,\parameters)$ is constant over $\parameters\in\paramspace$.
For a linear kernel, the regularizer reduces to standard $\ell_2$ regularization around the initialization.

\paragraph{Extension of guarantees.}
The guarantees in \autoref{sec:error-bounds} and the subsequent regret analysis extend to this formulation by conducting the same Hilbert-space argument in $\Hspace_\paramspace$. Indeed, setting $\weight_\parameters := \kernel_\paramspace(\pch,\parameters)$, we have $\model(\location,\parameters)=\inner{\weight_\parameters,\model(\location,\pch)}_{\Hspace_\paramspace}$, while the regularizer becomes the squared Hilbert-space distance between $\weight_\parameters$ and $\weight_{\parameters_{\iteridx,0}}$. The induced input-space kernel remains $\kernel_\mdlspace(\location,\location')=\inner{\model(\location,\pch),\model(\location',\pch)}_{\Hspace_\paramspace}$, preserving the concentration and predictive-variance arguments. The main adjustment is to impose \autoref{a:prior} on $d_\paramspace$ rather than $\paramdist$, requiring
\begin{equation}
    \prob{d_\paramspace(\parameters_\star,\parameters_0)\leq\pradius} \geq \anyconstant_\star\pradius^\deff, \qquad
    \prob{d_\paramspace(\parameters_\star,\parameters_0)>\pfunction(\anyscalar)}\leq e^{-\anyscalar},
\end{equation}
for the same ranges of $\pradius$ and $\anyscalar$ as in \autoref{a:prior}. Under these modified conditions, the variance-decay and regret guarantees retain their asymptotic forms, up to differences in constant factors.

\paragraph{Empirical illustrations.}
We illustrate the theoretical mechanisms using a nonlinear parametric surrogate trained with kernel-based regularization. The objective is defined on a grid of $41$ points in $[-1,1]$ and is generated by a fixed three-layer $\tanh$ neural network of width $32$. We use one initial observation and $300$ sequential rounds, with Gaussian observation noise of standard deviation $0.5$. At each round, a fresh network is initialized with a scaled LeCun-normal prior (scale $8$) and fitted for $250$ Adam steps using the logarithmic schedule $\regfactor_\iteridx = 0.25[1+\log(1+\iteridx)]^2$. Ten independent policy seeds are generated from a common master seed. We measure the fitted model and a fixed Matérn diagnostic GP (lengthscale $0.05$ and ridge $0.25$), the number of exact visits to $\location^\star$, and cumulative regret $\Regret_\iteridx$; curves below show individual policy seeds and their aggregate means.\footnote{Code for empirical illustrations and additional analyses will be made available at: \url{https://github.com/csiro-funml/nonlinear-bo-theory}.}

\autoref{fig:empirical-illustration} shows that the randomized regularized policy repeatedly selects the grid optimum in most runs. Each query maximizes a single fitted network, and the policy uses no explicit uncertainty estimate. Recurrence varies across seeds, with some runs making few or no visits to the optimum within the observation budget. Runs with frequent visits generally attain lower latent GP variance at the optimum, illustrating the recurrence mechanism underlying \autoref{thr:variance-decay}. Additional experiments exploring the effect of regularization schedules, kernels and priors are presented in \autoref{sec:ablation}.

\begin{figure*}[t]
    \centering
    \begin{subfigure}[t]{0.32\textwidth}
        \centering
        \includegraphics[width=\linewidth]{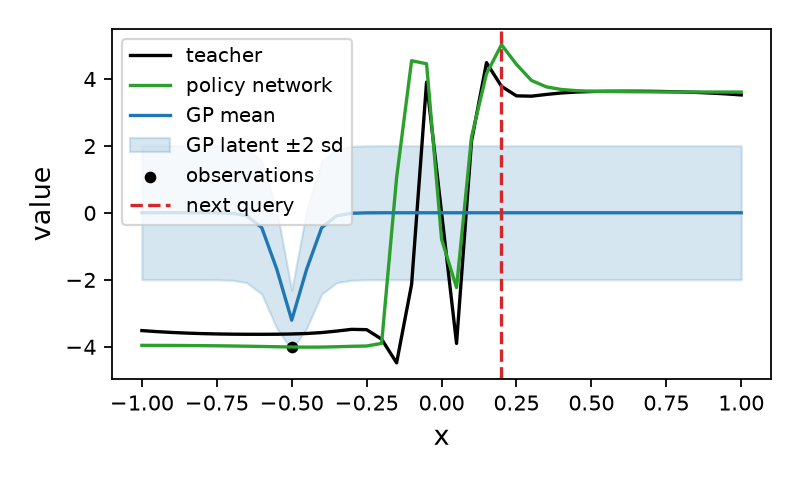}
        \caption{$n=1$}
    \end{subfigure}\hfill
    \begin{subfigure}[t]{0.32\textwidth}
        \centering
        \includegraphics[width=\linewidth]{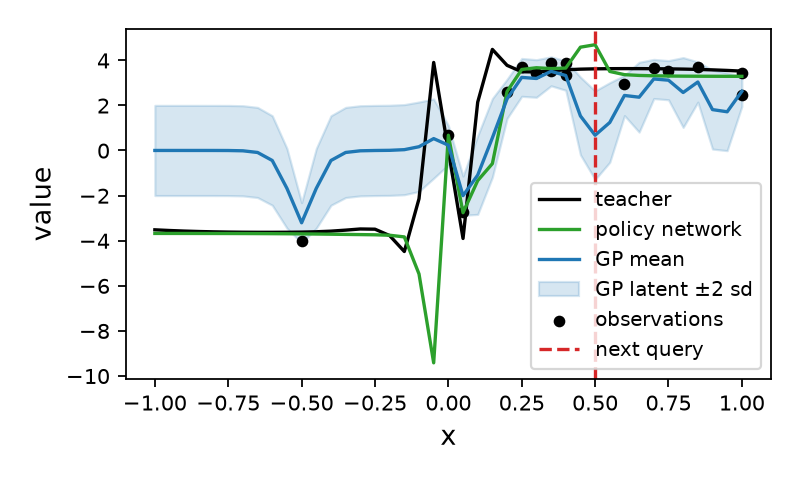}
        \caption{$n=16$}
    \end{subfigure}\hfill
    \begin{subfigure}[t]{0.32\textwidth}
        \centering
        \includegraphics[width=\linewidth]{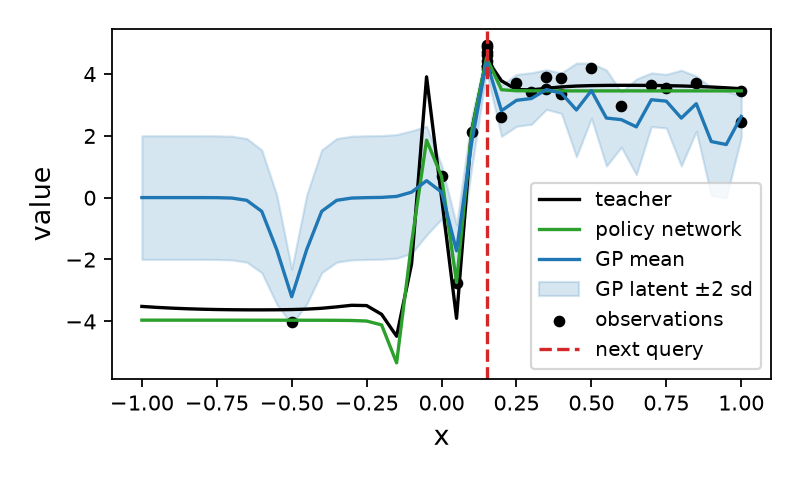}
        \caption{$n=26$}
    \end{subfigure}

    \begin{subfigure}[t]{0.32\textwidth}
        \centering
        \includegraphics[width=\linewidth]{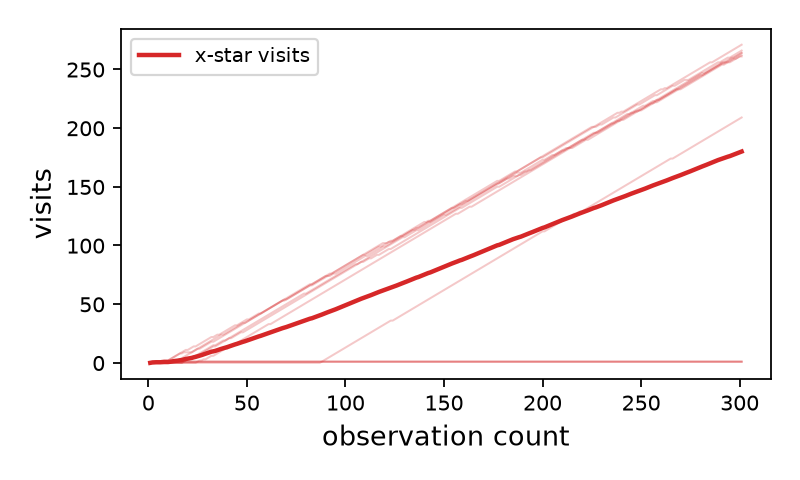}
        \caption{optimum visits}
    \end{subfigure}\hfill
    \begin{subfigure}[t]{0.32\textwidth}
        \centering
        \includegraphics[width=\linewidth]{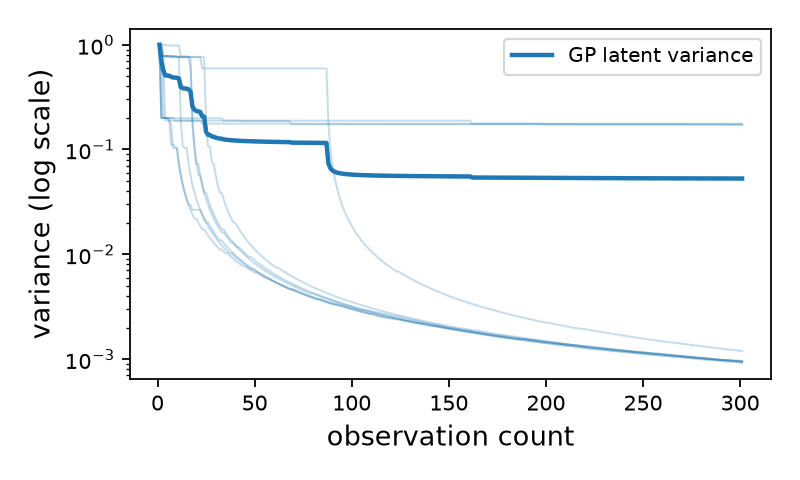}
        \caption{GP variance}
    \end{subfigure}\hfill
    \begin{subfigure}[t]{0.32\textwidth}
        \centering
        \includegraphics[width=\linewidth]{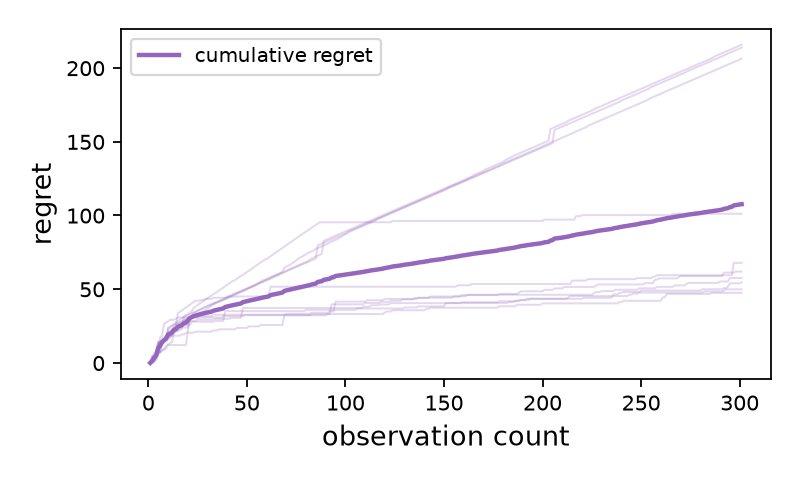}
        \caption{cumulative regret}
    \end{subfigure}
    \caption{Empirical illustration of randomized regularized optimization on the finite grid. The top row shows representative fitted-model snapshots at three observation counts, with the teacher, policy network, diagnostic GP mean and latent uncertainty, observations, and next query. The bottom row shows individual policy-seed traces and aggregate means for exact visits to the optimum, auxiliary GP variance at the optimum, and cumulative regret.}
    \label{fig:empirical-illustration}
\end{figure*}

\section{Related work}
Our work is related to regret analyses for GP bandits and Bayesian optimization, including GP-UCB \citep{Srinivas2010} and Thompson sampling \citep{Russo2014} analyses, where uncertainty is represented directly through a Gaussian process or kernel model. It also connects to neural bandit \citep{Phan-Trong2023} and neural Thompson sampling \citep{Zhang2021nts} results, which typically obtain guarantees through linearized models, neural tangent kernels \citep{Jacot2018}, or other infinite-width approximations \citep{Matthews2018}. In contrast, our analysis treats finite-dimensional nonlinear parametric models through an RKHS $\Hspace_\mdlspace$ induced by the parameter-space kernel $\kernel_\paramspace$ and studies randomized regularized training under adaptive data collection. The work is also motivated by likelihood-free Bayesian optimization \citep{Song2022a}, including BORE \citep{Tiao2021} and related density-ratio \citep{Oliveira2022} or variational search-distribution approaches \citep{Steinberg2024}, where acquisition functions are learned directly from utility-labeled data rather than derived from an explicit GP surrogate.
To the best of our knowledge, the only existing cumulative regret analysis for LFBO is that of \citet{Oliveira2022}, which concerns kernel-based classifiers rather than nonlinear parametric models. Their BORE regret bound assumes predictive-variance decay at the optimum, while BORE++ introduces explicit GP-UCB exploration. In contrast, our analysis establishes variance decay through recurring favorable random initializations preserved by regularization, without requiring a GP or kernel-based predictive model.

The population-level equivalence results of the LFBO framework in \citet{Song2022a} rely on consistency assumptions for statistical estimators trained on non-\iid sequential data. Such assumptions are nontrivial in adaptive optimization, where the sampling distribution depends on the learned acquisition model and may concentrate prematurely. Our work makes significant progress toward addressing this gap by deriving finite-time generalization bounds and identifying sufficient exploration and stability conditions for provable optimization guarantees in LFBO.

\section{Discussion and limitations}

We have developed a kernel-based framework for analyzing nonlinear parametric models trained under adaptive data collection. The main technical idea is to use a kernel over the parameters space to induce an RKHS geometry over the model class, which allows for regularized nonlinear estimators to be controlled through concentration arguments in Hilbert spaces. As a result, we obtain GP-style predictive-variance bounds for models that are not themselves Gaussian processes or kernel machines, and that provides a way to reuse information-gain tools in settings involving finite-dimensional nonlinear models without requiring their linearization. In the randomized-policy setting, the framework also clarifies how random initialization and explicit regularization can jointly induce exploration: favorable initializations occur with a rate controlled by the prior small-ball probability, while regularization ensures that such initializations are not overwhelmed by noise.

Our regret analysis instantiates this framework for finite-domain approximate Thompson sampling, establishing a sublinear cumulative regret bound for nonlinear parametric models trained with regularization around random initializations, without relying on GP surrogates or linearized approximations.
The proof identifies a concrete stability-recurrence tradeoff: increasing $\regfactor_\iteridx$ stabilizes the trained model around its initialization, while also reducing the probability of drawing an initialization sufficiently close to the target for the convergence argument to apply.
The resulting bound remains sublinear, with an additional logarithmic factor depending on the effective small-ball dimension \(\deff\). This dimension is expressed in the RKHS-induced pseudometric, rather than being fixed a priori as the ambient parameter dimension, leaving room for sharper problem-dependent characterizations.

Several extensions remain open. The formal regret theorem considers finite-domain, realizable, unique-optimum approximate Thompson sampling, focusing on the more immediate problem of controlling approximation errors of nonlinear randomized models under adaptive sampling. General LFBO regret guarantees would additionally require acquisition-specific alignment arguments, ensuring that learned acquisition maximizers eventually agree with those of the objective. This is particularly delicate when changing utility thresholds relabel past observations and affect predictability. Lastly, extending the framework to continuous domains would also require local variance-reduction arguments in place of exact visits to $\location^\star$. We view these as natural next steps enabled by the general RKHS-based approximation analysis developed here.

\bibliographystyle{plainnat}
\bibliography{references}

\clearpage
\appendix

\section{Auxiliary results}
\begin{definition}[Sub-Gaussianity]
	\label{def:sub-g}
	A real-valued random variable $\vnoise$ taking values is said to be $\sigma_\vnoise^2$-sub-Gaussian, given $\sigma_\vnoise>0$, if:
	\begin{equation}
		\forall \anyscalar\in\R, \quad \expectation[\exp(\anyscalar \vnoise)] \leq \exp \left( \frac{1}{2} \anyscalar^2\sigma_\vnoise^2 \right).
	\end{equation}
	Likewise, a real-valued stochastic process $\{\vnoise_\nobs\}_{\nobs=1}^\infty$ adapted to a filtration $\{\filtration_\nobs\}_{\nobs=0}^\infty$ is conditionally $\vncov$-sub-Gaussian if the following almost surely holds:
	\begin{equation}
		\forall \anyscalar\in\R, \quad \expectation[\exp(\anyscalar \vnoise) \mid \filtration_{\iteridx-1}] \leq \exp \left( \frac{1}{2} \anyscalar^2\sigma_\vnoise^2 \right), \quad \forall \nobs \in \N.
	\end{equation}
\end{definition}

\begin{lemma}[{\citealp[Cor. 3.6]{Abbasi-Yadkori2012}}]
	\label{thr:noise-bound}
	Let $\{\filtration_\iteridx\}_{\iteridx=0}^\infty$ be an increasing filtration, $\{\obsnoise_\iteridx\}_{\iteridx=1}^\infty$ be a real-valued stochastic process, and $\{\anyfunction_\iteridx\}_{\iteridx=1}^\infty$ be a stochastic process taking values in a separable real Hilbert space $\Hspace$, with both processes adapted to the filtration. Assume that $\{\anyfunction_\iteridx\}_{\iteridx=1}^\infty$ is also predictable, i.e.,  $\anyfunction_\iteridx$ is $\filtration_{\iteridx-1}$-measurable, and that $\obsnoise_\iteridx$ is conditionally $\sigma_\obsnoise^2$-sub-Gaussian, for all $\iteridx \in \N$. Then, given any $\delta \in (0,1)$, with probability at least $1-\delta$,
	\begin{equation*}
		\forall \iteridx \in \N, \quad
		\norm*{\sum_{i=1}^\iteridx \obsnoise_i \anyfunction_i}_{
			(\operator{V} + \anyfunctions_\iteridx\anyfunctions_\iteridx^\transpose)^{-1}
		}^2
		\leq 2\sigma_\obsnoise^2 \log \left(
		\frac{
			\det(\eye + \anyfunctions_\iteridx^\transpose \operator{V}^{-1}\anyfunctions_\iteridx)^{\frac{1}{2}}
		}{\delta}
		\right),
	\end{equation*}
	for any positive-definite operator $\operator{V} \succ 0$ on $\Hspace$, and where we set $\anyfunctions_\iteridx := [\anyfunction_1, \dots, \anyfunction_\iteridx]$.
\end{lemma}

\section{Proofs of main results}
\label{sec:proofs}

\subsection{Proof of \autoref{thr:model-matern}}
\begin{proof}
    Let $m:=\nu+d/2$. The RKHS of a Matérn kernel on $\domain\subset\mathbb R^d$ with smoothness parameter $\nu$ is norm-equivalent to the Sobolev space $H^m(\domain)$, up to the usual restriction to the compact domain $\domain$ \citep{Kanagawa2025}. Hence, it is enough to show that every $h\in \Hspace_\mdlspace$ has Sobolev norm controlled by its $\Hspace_\mdlspace$-norm.

    By Lemma~3.1, for every $h\in \Hspace_\mdlspace$ and every $\eta>0$, there exists $w_\eta\in \Hspace_\paramspace$ such that
    \begin{equation}
    h(\location)
    =
    \inner*{w_\eta, g(\location,\cdot)}_{\Hspace_\paramspace},
    \qquad
    \forall \location\in\domain,
    \end{equation}
    and
    \begin{equation}
    \norm{w_\eta}_{\Hspace_\paramspace}
    \le
    \norm{h}_{\Hspace_\mdlspace}+\eta.
    \end{equation}
    For every multi-index $\vec{\alpha}$ with $\norm{\vec{\alpha}}_1\le s$, differentiating under the inner product gives
    \begin{equation}
    D^\vec{\alpha}_\location h(\location)
    =
    \inner*{
    w_\eta,
    D^\vec{\alpha}_\location g(\location,\cdot)
    }_{\Hspace_\paramspace}.
    \end{equation}
    Therefore, by Cauchy-Schwarz and the assumption of the lemma,
    \begin{equation}
    |D^\vec{\alpha}_\location h(\location)|
    \le
    \norm{w_\eta}_{\Hspace_\paramspace}
    \sup_{\location\in\domain}
    \norm*{D^\vec{\alpha}_\location g(\location,\cdot)}_{\Hspace_\paramspace}
    \le
    C_\vec{\alpha}\left(\norm{h}_{\Hspace_\mdlspace}+\eta\right),
    \end{equation}
    where
    \begin{equation}
    C_\vec{\alpha}
    :=
    \sup_{\location\in\domain}
    \norm*{D^\vec{\alpha}_\location g(\location,\cdot)}_{\Hspace_\paramspace}
    <\infty.
    \end{equation}
    Since $\domain$ is compact, this implies
    \begin{equation}
    \norm{D^\vec{\alpha} h}_{L^2(\domain)}
    \le
    \vol{\domain}^{1/2}C_\vec{\alpha}\left(\norm{h}_{\Hspace_\mdlspace}+\eta\right),
    \end{equation}
    where $\vol{\domain}$ is the volume of the domain, i.e., the Lebesgue measure of $\domain \subset \R^\locdim$.
    Summing over all $\vec{\alpha}$ with $\norm{\vec{\alpha}}_1\le \anyscalar$, we obtain
    \begin{equation}
    \norm{h}_{H^s(\domain)}
    \le
    C_\anyscalar\left(\norm{h}_{\Hspace_\mdlspace}+\eta\right),
    \end{equation}
    for a finite constant $C_\anyscalar>0$. Letting $\eta\downarrow0$ yields
    \begin{equation}
    \norm{h}_{H^s(\domain)}
    \le
    C_\anyscalar\norm{h}_{\Hspace_\mdlspace}.
    \end{equation}

    Because $s\ge \nu+d/2=m$, the Sobolev-scale inclusion gives a continuous embedding
    \begin{equation}
    H^s(\domain)\subseteq H^m(\domain).
    \end{equation}
    Since $H^m(\domain)$ is norm-equivalent to the RKHS $\Hspace_\kernel$ of the Matérn kernel $\kernel$, there exists a constant $C_\kernel>0$ such that
    \begin{equation}
    \norm{h}_{\Hspace_\kernel}
    \le
    C_\kernel\norm{h}_{H^s(\domain)}
    \le
    C_\kernel C_\anyscalar\norm{h}_{\Hspace_\mdlspace}.
    \end{equation}
    Thus $\Hspace_\mdlspace\subseteq \Hspace_\kernel$ continuously.

    Finally, by the RKHS inclusion theorem of \citet[Thm. II, Sec. 7]{Aronszajn1950}, the continuous embedding
    \begin{equation}
    \Hspace_\mdlspace\subseteq \Hspace_\kernel,
    \qquad
    \norm{h}_{\Hspace_\kernel}\le b_\mdlspace\norm{h}_{\Hspace_\mdlspace},
    \end{equation}
    with $b_\mdlspace:=C_\kernel C_\anyscalar<\infty$, implies the kernel domination
    \begin{equation}
    k_\mdlspace\preceq b_\mdlspace^2\kernel.
    \end{equation}
    This proves both $\Hspace_\mdlspace\subseteq \Hspace_\kernel$ and $k_\mdlspace\preceq b_\mdlspace^2\kernel$.
\end{proof}

\subsection{Proof of \autoref{thr:error-bound}}
\begin{proof}
Let
\begin{equation}
\widehat{\model}_\iteridx
\in
\argmin_{\model\in\Hspace_\mdlspace}
L_\iteridx(\model)
\end{equation}
denote the minimizer of the regularized loss over the full RKHS, where
\begin{equation}
L_\iteridx(\model)
:=
\frac{\regfactor_\iteridx}{2}
\norm{\model-\model_{\iteridx,0}}_\mdlspace^2
+
\sum_{i=1}^{\iteridx}
\loss(\model(\location_i),\clabel_i).
\end{equation}
Also write
\begin{equation}
\model_\star:=\model(\placeholder,\parameters_\star),
\qquad
\Delta_\iteridx:=\widehat{\model}_\iteridx-\model_\star .
\end{equation}
The proof first controls $\widehat{\model}_\iteridx-\model_\star$ in the
$\Hessian_\iteridx$-norm, and then converts this into a pointwise bound.

Since $\widehat{\model}_\iteridx$ minimizes $L_\iteridx$ over
$\Hspace_\mdlspace$, the first-order optimality condition gives
\begin{equation}
\nabla L_\iteridx(\widehat{\model}_\iteridx)=0.
\end{equation}
Therefore,
\begin{equation}
\nabla L_\iteridx(\widehat{\model}_\iteridx)
-
\nabla L_\iteridx(\model_\star)
=
-\nabla L_\iteridx(\model_\star).
\end{equation}
Taking the inner product with $\Delta_\iteridx$ gives
\begin{equation}
\inner*{
\Delta_\iteridx,
\nabla L_\iteridx(\widehat{\model}_\iteridx)
-
\nabla L_\iteridx(\model_\star)
}_\mdlspace
=
-\inner*{\Delta_\iteridx,\nabla L_\iteridx(\model_\star)}_\mdlspace.
\end{equation}

We now lower-bound the left-hand side. The Hessian of the regularizer is
$\regfactor_\iteridx\eye$. Moreover, by \autoref{a:loss} and the reproducing
property,
\begin{equation}
\nabla^2
\loss(\model(\location_i),\clabel_i)
=
\ddot\loss_{\clabel_i}(\model(\location_i))
\feature(\location_i)\otimes\feature(\location_i)
\succeq
\lossfactor
\feature(\location_i)\otimes\feature(\location_i).
\end{equation}
Hence, along the line segment between $\model_\star$ and
$\widehat{\model}_\iteridx$,
\begin{equation}
\nabla^2 L_\iteridx(\model)
\succeq
\regfactor_\iteridx\idop
+
\lossfactor
\sum_{i=1}^{\iteridx}
\feature(\location_i)\otimes\feature(\location_i)
=:
\Hessian_\iteridx .
\end{equation}
By strong convexity of \(L_\iteridx\) with curvature operator \(\Hessian_\iteridx\),
\begin{align}
&
\inner*{
\Delta_\iteridx,
\nabla L_\iteridx(\widehat{\model}_\iteridx)
-
\nabla L_\iteridx(\model_\star)
}_\mdlspace
\nonumber\\
&\qquad =
\int_0^1
\inner*{
\Delta_\iteridx,
\nabla^2 L_\iteridx(\model_\star+s\Delta_\iteridx)\Delta_\iteridx
}_\mdlspace
\,ds
\nonumber\\
&\qquad \ge
\norm{\Delta_\iteridx}_{\Hessian_\iteridx}^2 .
\end{align}
Combining this with the previous identity yields
\begin{equation}
\norm{\Delta_\iteridx}_{\Hessian_\iteridx}^2
\le
-\inner*{\Delta_\iteridx,\nabla L_\iteridx(\model_\star)}_\mdlspace
\le
\norm{\Delta_\iteridx}_{\Hessian_\iteridx}
\norm{\nabla L_\iteridx(\model_\star)}_{\Hessian_\iteridx^{-1}}.
\end{equation}
Therefore,
\begin{equation}
\norm{\widehat{\model}_\iteridx-\model_\star}_{\Hessian_\iteridx}
\le
\norm{\nabla L_\iteridx(\model_\star)}_{\Hessian_\iteridx^{-1}}.
\end{equation}

Expanding the gradient at $\model_\star$ gives
\begin{equation}
\nabla L_\iteridx(\model_\star)
=
\regfactor_\iteridx(\model_\star-\model_{\iteridx,0})
+
\sum_{i=1}^{\iteridx}
\dot\loss_{\clabel_i}(\model_\star(\location_i))
\feature(\location_i).
\end{equation}
Consequently,
\begin{align}
\norm{\widehat{\model}_\iteridx-\model_\star}_{\Hessian_\iteridx}
&\le
\norm*{
\regfactor_\iteridx(\model_\star-\model_{\iteridx,0})
}_{\Hessian_\iteridx^{-1}}
+
\norm*{
\sum_{i=1}^{\iteridx}
\dot\loss_{\clabel_i}(\model_\star(\location_i))
\feature(\location_i)
}_{\Hessian_\iteridx^{-1}}
\nonumber\\
&\le
\sqrt{\regfactor_\iteridx}
\norm{\model_\star-\model_{\iteridx,0}}_\mdlspace
+
\norm*{
\sum_{i=1}^{\iteridx}
\dot\loss_{\clabel_i}(\model_\star(\location_i))
\feature(\location_i)
}_{\Hessian_\iteridx^{-1}},
\end{align}
where the last inequality follows from
$\Hessian_\iteridx\succeq \regfactor_\iteridx\idop$.

Define
\begin{equation}
\error_i
:=
\dot\loss_{\clabel_i}(\model_\star(\location_i)).
\end{equation}
By assumption, $\error_i$ is conditionally $\sigma_\loss^2$-sub-Gaussian
given $\filtration_{i-1}$, and $\feature(\location_i)$ is predictable. Since
$\regfactor_0\leq \regfactor_\iteridx$, we have
\begin{equation}
\Hessian_\iteridx
\succeq
\regfactor_0\idop
+
\lossfactor
\sum_{i=1}^{\iteridx}\feature(\location_i)\otimes\feature(\location_i).
\end{equation}
Therefore,
\begin{equation}
\norm*{
\sum_{i=1}^{\iteridx}
\error_i\feature(\location_i)
}_{\Hessian_\iteridx^{-1}}
\le
\norm*{
\sum_{i=1}^{\iteridx}
\error_i\feature(\location_i)
}_{\left(\regfactor_0\idop+
\lossfactor\sum_{i=1}^{\iteridx}\feature(\location_i)\otimes\feature(\location_i)\right)^{-1}}.
\end{equation}
Applying \autoref{thr:noise-bound} with the fixed operator
$V=\regfactor_0\idop$ and
$h_i=\sqrt{\lossfactor}\feature(\location_i)$ yields, with probability at
least $1-\delta$, simultaneously for all $\iteridx\in\N$,
\begin{equation}
\norm*{
\sum_{i=1}^{\iteridx}
\error_i\feature(\location_i)
}_{\Hessian_\iteridx^{-1}}
\le
\sigma_\loss
\sqrt{
2\lossfactor^{-1}
\log\left(
\frac{
\det(\eye+\lossfactor\regfactor_0^{-1}\mat{\kernel}_\iteridx^\mdlspace)^{1/2}
}{\delta}
\right)
}.
\end{equation}
Thus,
\begin{equation}
\norm{\widehat{\model}_\iteridx-\model_\star}_{\Hessian_\iteridx}
\le
\widehat\beta_\iteridx(\delta),
\end{equation}
where
\begin{equation}
\widehat\beta_\iteridx(\delta)
=
\sqrt{\regfactor_\iteridx}
\norm{\model_\star-\model_{\iteridx,0}}_\mdlspace
+
\sigma_\loss
\sqrt{
2\lossfactor^{-1}
\log\left(
\frac{
\det(\eye+\lossfactor\regfactor_0^{-1}\mat{\kernel}_\iteridx^\mdlspace)^{1/2}
}{\delta}
\right)
}.
\end{equation}

For any $\location\in\domain$, by the reproducing property and
Cauchy-Schwarz,
\begin{equation}
\abs{\widehat{\model}_\iteridx(\location)-\model_\star(\location)}
=
\abs{\inner*{
\widehat{\model}_\iteridx-\model_\star,
\feature(\location)
}_\mdlspace}
\le
\norm{\widehat{\model}_\iteridx-\model_\star}_{\Hessian_\iteridx}
\norm{\feature(\location)}_{\Hessian_\iteridx^{-1}}.
\end{equation}
Therefore,
\begin{equation}
\abs{\widehat{\model}_\iteridx(\location)-\model_\star(\location)}
\le
\widehat\beta_\iteridx(\delta)
\norm{\feature(\location)}_{\Hessian_\iteridx^{-1}}.
\end{equation}

It remains to pass from the RKHS minimizer \(\widehat{\model}_\iteridx\) to the
deployed parametric model \(\model_\iteridx\). Since
\(\parameters_\iteridx\) minimizes \autoref{eq:random-loss} over the
parametric class and \(\model_\star=\model(\placeholder,\parameters_\star)\)
belongs to this class, we have
\begin{equation}
L_\iteridx(\model_\iteridx)
\le
L_\iteridx(\model_\star).
\end{equation}
Moreover, since \(\widehat{\model}_\iteridx\) minimizes \(L_\iteridx\) over
\(\Hspace_\mdlspace\), the strong convexity lower bound gives
\begin{equation}
\frac{1}{2}
\norm{\model_\iteridx-\widehat{\model}_\iteridx}_{\Hessian_\iteridx}^2
\le
L_\iteridx(\model_\iteridx)
-
L_\iteridx(\widehat{\model}_\iteridx).
\end{equation}
Combining the last two displays,
\begin{equation}
\frac{1}{2}
\norm{\model_\iteridx-\widehat{\model}_\iteridx}_{\Hessian_\iteridx}^2
\le
L_\iteridx(\model_\star)
-
L_\iteridx(\widehat{\model}_\iteridx).
\end{equation}
Applying the upper loss-geometry bound at \(\model_\star\), we also have
\begin{equation}
L_\iteridx(\model_\star)
-
L_\iteridx(\widehat{\model}_\iteridx)
\le
\frac{1}{2}
\norm{\nabla L_\iteridx(\model_\star)}_{\Hessian_\iteridx^{-1}}^2.
\end{equation}
Therefore,
\begin{equation}
\norm{\model_\iteridx-\widehat{\model}_\iteridx}_{\Hessian_\iteridx}
\le
\norm{\nabla L_\iteridx(\model_\star)}_{\Hessian_\iteridx^{-1}}.
\end{equation}
On the concentration event derived above,
\begin{equation}
\norm{\nabla L_\iteridx(\model_\star)}_{\Hessian_\iteridx^{-1}}
\le
\widehat\beta_\iteridx(\delta),
\end{equation}
and hence
\begin{equation}
\norm{\model_\iteridx-\widehat{\model}_\iteridx}_{\Hessian_\iteridx}
\le
\widehat\beta_\iteridx(\delta).
\end{equation}

For any \(\location\in\domain\), the reproducing property and
Cauchy--Schwarz give
\begin{equation}
\abs{\model_\iteridx(\location)-\widehat{\model}_\iteridx(\location)}
\le
\norm{\model_\iteridx-\widehat{\model}_\iteridx}_{\Hessian_\iteridx}
\norm{\feature(\location)}_{\Hessian_\iteridx^{-1}}
\le
\widehat\beta_\iteridx(\delta)
\norm{\feature(\location)}_{\Hessian_\iteridx^{-1}}.
\end{equation}
The RKHS-minimizer part of the proof already gives
\begin{equation}
\abs{\widehat{\model}_\iteridx(\location)-\model_\star(\location)}
\le
\widehat\beta_\iteridx(\delta)
\norm{\feature(\location)}_{\Hessian_\iteridx^{-1}}.
\end{equation}
Thus, by the triangle inequality,
\begin{equation}
\abs{\model_\iteridx(\location)-\model_\star(\location)}
\le
2\widehat\beta_\iteridx(\delta)
\norm{\feature(\location)}_{\Hessian_\iteridx^{-1}},
\qquad
\forall \location\in\domain,\ \forall \iteridx\ge 0.
\end{equation}
This proves the claim.
\end{proof}

\subsection{Proof of \autoref{thr:error-gp}}
\begin{proof}
    Let
    \begin{equation}
    \Hessian_\iteridx
    =
    \regfactor_\iteridx \idop
    +
    \lossfactor
    \sum_{i=1}^{\iteridx}
    \feature_\mdlspace(\location_i)\otimes \feature_\mdlspace(\location_i),
    \end{equation}
    where $\feature_\mdlspace(\location):=k_\mdlspace(\cdot,\location)$. Since $\lambda_0\le \regfactor_\iteridx$, we have
    \begin{equation}
    \Hessian_\iteridx
    \succeq
    \lambda_0 \idop
    +
    \lossfactor
    \sum_{i=1}^{\iteridx}
    \feature_\mdlspace(\location_i)\otimes \feature_\mdlspace(\location_i).
    \end{equation}
    Therefore,
    \begin{equation}
    \Hessian_\iteridx^{-1}
    \preceq
    \left(
    \lambda_0 \idop
    +
    \lossfactor
    \sum_{i=1}^{\iteridx}
    \feature_\mdlspace(\location_i)\otimes \feature_\mdlspace(\location_i)
    \right)^{-1}.
    \end{equation}
    It follows that
    \begin{equation}
    \|\feature_\mdlspace(\location)\|_{\Hessian_\iteridx^{-1}}^2
    \le
    \left\langle
    \feature_\mdlspace(\location),
    \left(
    \lambda_0 \idop
    +
    \lossfactor
    \sum_{i=1}^{\iteridx}
    \feature_\mdlspace(\location_i)\otimes \feature_\mdlspace(\location_i)
    \right)^{-1}
    \feature_\mdlspace(\location)
    \right\rangle_\mdlspace .
    \end{equation}
    By Woodbury's identity,
    \begin{equation}
    \|\feature_\mdlspace(\location)\|_{\Hessian_\iteridx^{-1}}^2  \leq \lambda_0^{-1}
    \left[
    k_\mdlspace(\location,\location)
    -
    \vec\kernel_{\mdlspace,\iteridx}^\top
    \left(
    \Kernel^\mdlspace_\iteridx
    +
    \lambda_0\lossfactor^{-1}\eye
    \right)^{-1}
    \vec\kernel_{\mdlspace,\iteridx}
    \right],
    \end{equation}
    where
    $\vec\kernel_{\mdlspace,\iteridx}:=[k_\mdlspace(\location,\location_i)]_{i=1}^{\iteridx}$ and
    $\Kernel^\mdlspace_\iteridx:=[k_\mdlspace(\location_i,\location_j)]_{i,j=1}^{\iteridx}$.

    We now compare the posterior variances induced by $k_\mdlspace$ and $\kernel$. Let
    $r:=\lambda_0\lossfactor^{-1}$. For any positive-semidefinite kernel
    $\widetilde\kernel$, the GP posterior variance at $\location$ with noisy observations at
    $\location_1,\ldots,\location_\iteridx$ and ridge $r$ is the Schur complement
    \begin{equation}
    \sigma_{\widetilde\kernel,\iteridx}^2(\location)
    =
    \widetilde\kernel(\location,\location)
    -
    \vec{\widetilde\kernel}_\iteridx(\location)^\top
    (\mat{\widetilde K}_\iteridx+r\eye)^{-1}
    \vec{\widetilde\kernel}_\iteridx(\location).
    \end{equation}
    Since $\kernel-k_\mdlspace$ is positive semidefinite, the joint covariance matrix over
    $(\location,\location_1,\ldots,\location_\iteridx)$ induced by $\kernel$ dominates the
    corresponding joint covariance matrix induced by $k_\mdlspace$. Hence the conditional variance,
    given by the Schur complement, is also larger:
    \begin{equation}
    \sigma_{k_\mdlspace,\iteridx}^2(\location)
    \le
    \sigma_{\kernel,\iteridx}^2(\location).
    \end{equation}

    It remains to compare the log-determinant term. Since $\kernel\succeq k_\mdlspace$ and $\lambda_0\le \regfactor_\iteridx$,
    \begin{equation}
    \lossfactor\regfactor_\iteridx^{-1}\Kernel^\mdlspace_\iteridx
    \preceq
    \lossfactor\lambda_0^{-1}\Kernel_\iteridx.
    \end{equation}
    By monotonicity of the determinant over positive-semidefinite matrices,
    \begin{equation}
    \log\det\left(
    \eye+\lossfactor\regfactor_\iteridx^{-1}\Kernel^\mdlspace_\iteridx
    \right)
    \le
    \log\det\left(
    \eye+\lossfactor\lambda_0^{-1}\Kernel_\iteridx
    \right).
    \end{equation}

    Combining this with \autoref{thr:error-bound}, on the same event of probability at least $1-\delta$,
    \begin{align}
    |g(\location,\parameters_\iteridx)-g(\location,\parameters_\star)|
    &\le
    2\widehat\beta_\iteridx(\delta)\|\feature_\mdlspace(\location)\|_{\Hessian_\iteridx^{-1}} \\
    &\le
    2\lambda_0^{-1/2}\widehat\beta_\iteridx(\delta)\sigma_\iteridx(\location) \\
    &\le
    2\beta_\iteridx(\delta)\sigma_\iteridx(\location),
    \end{align}
    where
    \begin{equation}
    \beta_\iteridx(\delta)
    =
    \sqrt{\frac{\regfactor_\iteridx}{\lambda_0}}
    \|\model_\star-\model_{\iteridx,0}\|_\mdlspace
    +
    \sqrt{
    \frac{2\sigma_\ell^2}{\lossfactor\lambda_0}
    \log\left(
    \frac{\det(\eye+\lossfactor\lambda_0^{-1}\Kernel_\iteridx)^{1/2}}{\delta}
    \right)
    }.
    \end{equation}
    This proves the claim.
\end{proof}

\subsection{Proof of \autoref{thr:variance-decay}}
\begin{proof}
Let us define a sequence of favorable initialization events and the high-probability event of \autoref{thr:error-bound}, respectively, as:
\begin{align}
\event_\iteridx &:=
\left\{
\distance(\parameters_\star,\parameters_{\iteridx,0}) \le \regfactor_\iteridx^{-1/2}
\right\}, \quad \iteridx\geq 0,\\
\event_G(\delta) &:= \braces{
    \forall\iteridx\geq 0, \, \forall \location\in\domain,
    \: \abs{\model_\iteridx(\location) - \model_\star(\location)} \leq 2\widehat\beta_\iteridx(\delta) \norm{\feature(\location)}_{\Hessian_\iteridx^{-1}}
}.\label{eq:good-error-event}
\end{align}
By \autoref{a:prior}, for all sufficiently large $\iteridx$,
\begin{equation}
\prob{\event_\iteridx}
\ge
c_\star \regfactor_\iteridx^{-\deff/2}.
\end{equation}
Since $\parameters_{\iteridx,0}\sim\paramprior$ is sampled independently of the past at round $\iteridx$, the same lower bound holds conditionally on the past before drawing $\parameters_{\iteridx,0}$.

We first show that, on $\event_G(\delta)$, $\event_\iteridx$ eventually implies $\location_{\iteridx+1}=\location^\star$. On $\event_\iteridx$, the initialization term in $\widehat\beta_\iteridx(\delta)$ satisfies:
\begin{equation}
\sqrt{\regfactor_\iteridx}\|\model_\star-\model_{\iteridx,0}\|_\mdlspace
\le 1.
\end{equation}
Moreover, since $\domain$ is finite, the log-determinant term in $\widehat\beta_\iteridx(\delta)$ is $\bigo(\sqrt{\log \iteridx})$ (see \autoref{rem:beta-log-t}). Since $\Hessian_\iteridx \succeq \regfactor_\iteridx I$, we have:
\begin{equation}
\|\varphi(\location)\|_{\Hessian_\iteridx^{-1}}
\le
\frac{\|\varphi(\location)\|_\mdlspace}{\sqrt{\regfactor_\iteridx}}
\le
\frac{\kappa_\mdlspace}{\sqrt{\regfactor_\iteridx}},
\qquad
\kappa_\mdlspace^2:=\sup_{\location\in \domain} k_\mdlspace(\location,\location)<\infty.
\end{equation}
Therefore, on $\event_\iteridx$, $\widehat\beta_\iteridx$ is $\bigo(\sqrt{\log\iteridx})$, and additionally on $\event_G(\delta)$,
\begin{equation}
\sup_{\location\in \domain}|\model(\location,\parameters_\iteridx)-g(\location,\parameters_\star)|
\in
\bigo\!\left(\sqrt{\frac{\log \iteridx}{\regfactor_\iteridx}}\right).
\end{equation}
Because $\regfactor_\iteridx\in\omega(\log \iteridx)$, the right-hand side converges to zero. Since $\location^\star$ is the unique maximizer of $\objective=\model(\cdot,\parameters_\star)$ over the finite set $\domain$, the gap:
\begin{equation}
\Delta_\star :=
\objective(\location^\star)-\max_{\location\neq \location^\star} \objective(\location)
\end{equation}
is strictly positive. Hence, for all sufficiently large $\iteridx$,
\begin{equation}
\sup_{\location\in \domain}|\model(\location,\parameters_\iteridx)-\objective(\location)| < \Delta_\star/2,
\end{equation}
which implies that the maximizer of $\model(\cdot,\parameters_\iteridx)$ is $\location^\star$. Thus, eventually,
\begin{equation}
\event_G(\delta) \cap \event_\iteridx \subseteq \{\location_{\iteridx+1}=\location^\star\}.
\label{eq:optimum-event-implication}
\end{equation}

Define the event counters:
\begin{align}
    N_\iteridx^+ &:=  \sum_{i=1}^\iteridx \indicator[\event_i]\\
    N_\iteridx^\star &:=\sum_{i=1}^\iteridx \indicator[\location_i=\location^\star].
\end{align}
Since
\begin{equation}
\sum_{\iteridx=1}^\infty \regfactor_\iteridx^{-\deff/2}=\infty,
\end{equation}
an application of the second Borel-Cantelli lemma \citep[Thm. 4.5.5]{Durrett2019} yields:
\begin{equation}
    N_\iteridx^+ \in \Omega\!\left(\sum_{i=1}^\iteridx \regfactor_i^{-\deff/2}\right) \qquad \text{a.s.}
\end{equation}
Now the conclusion in \autoref{eq:optimum-event-implication} implies that there exists a sufficiently large $\iteridx_0$ such that, on $\event_G(\delta)$,
\begin{equation}
    \indicator[\location_{\iteridx+1} = \location^\star] \geq \indicator[\event_\iteridx], \quad \forall\iteridx \geq \iteridx_0.
\end{equation}
Consequently, for every $\iteridx\geq\iteridx_0$,
\begin{equation}
    \begin{split}
        N_{\iteridx+1}^\star = \sum_{i=1}^{\iteridx+1} \indicator[\location_i = \location^\star]
        \geq \sum_{i=\iteridx_0}^\iteridx \indicator[\event_i] = N_\iteridx^+ - N_{\iteridx_0-1}^+.
    \end{split}
\end{equation}
Since $N_{\iteridx_0-1}^+$ is finite and independent of $\iteridx$, it follows that, with probability at least $1-\delta$,
\begin{equation}
N_\iteridx^\star
\in
\Omega\!\left(
\sum_{i=1}^\iteridx \regfactor_i^{-\deff/2}
\right).
\end{equation}

It remains to translate visits to $\location^\star$ into predictive-variance decay. Since \autoref{thr:error-gp} uses the fixed-ridge GP variance:
\begin{equation}
\sigma_\iteridx^2(\location)
=
\kernel(\location,\location)-\vec\kernel_\iteridx(\location)^\top
(\Kernel_\iteridx+\regfactor_0\lossfactor^{-1}\eye)^{-1}\vec\kernel_\iteridx(\location),
\end{equation}
retaining only the $N_\iteridx^\star$ observations collected at $\location^\star$ can only increase the posterior variance. Writing $\kappa_\star^2:=\kernel(\location^\star,\location^\star)>0$, the variance at $\location^\star$ based only on these repeated observations is
\begin{equation}
\kappa_\star^2
-
\kappa_\star^4
\vec 1^\top
\left(
\kappa_\star^2\vec 1\vec 1^\top
+
\regfactor_0\lossfactor^{-1} \eye
\right)^{-1}
\vec 1
\ge \sigma_\iteridx^2(\location^\star)
,
\end{equation}
where $\vec 1\in\R^{N_\iteridx^\star}$. By the Sherman-Morrison identity, this equals
\begin{equation}
\frac{\kappa_\star^2\regfactor_0\lossfactor^{-1}}
{\regfactor_0\lossfactor^{-1}+\kappa_\star^2N_\iteridx^\star}
\in
\bigo\left(\frac{1}{N_\iteridx^\star}\right).
\end{equation}
Therefore,
\begin{equation}
\sigma_\iteridx^2(\location^\star)
\in
\bigo\left(
\left[
\sum_{i=1}^\iteridx \regfactor_i^{-\deff/2}
\right]^{-1}
\right).
\end{equation}

If $\regfactor_\iteridx\asymp \log^q \iteridx$, with $q>1$, then
\begin{equation}
\sum_{i=2}^\iteridx \regfactor_i^{-\deff/2}
\asymp
\sum_{i=2}^\iteridx (\log i)^{-q \deff/2}
\asymp
\frac{\iteridx}{(\log \iteridx)^{q \deff/2}}.
\end{equation}
Hence,
\begin{equation}
\sigma_\iteridx^2(\location^\star)
\in
\bigo\paren*{
\frac{(\log \iteridx)^{q \deff/2}}{\iteridx}
},
\end{equation}
as claimed.
\end{proof}

\subsection{Proof of \autoref{thr:ts-regret}}
\begin{proof}
    Our analysis will condition on the event $\event_G(\delta)$, defined in \autoref{eq:good-error-event}, and the favorable initializations event:
    \begin{align}
        \initevent(\delta) &:= \braces{
            \forall\iteridx \geq 0,
            \quad
            \paramdist(\parameters_\star, \parameters_{\iteridx,0}) \leq \distbound_\iteridx(\delta)
        },
    \end{align}
    where setting $\distbound_\iteridx(\delta) := \pfunction\paren*{\log \frac{\pi^2(\iteridx+1)^2}{6\delta}}$ (as in \autoref{a:prior}) yields $\prob{\initevent(\delta)} \geq 1 - \delta$, and, according to \autoref{thr:error-bound}, $\prob{\event_G(\delta)} \geq 1-\delta$. Then, by an union bound,
    \begin{equation}
        \prob{\initevent(\delta) \cap \event_G(\delta)} \geq 1 - 2\delta.
    \end{equation}
    Conditioning on $\event_G(\delta)$, \autoref{thr:error-gp} applies, and hence the regret decomposition in \autoref{sec:ts-analysis} leads us to:
    \begin{equation}
        \regret_\iteridx \leq 2\beta_{\iteridx-1}(\delta) (\sigma_{\iteridx-1}(\location^\star) + \sigma_{\iteridx-1}(\location_\iteridx)), \quad \forall\iteridx\in\N.
    \end{equation}
    For $\beta_\iteridx$, on $\initevent(\delta)$, we have that:
    \begin{equation}
        \begin{split}
            \beta_\iteridx(\delta)
            &\leq \beta_\iteridx^\star(\delta) := \distbound_\iteridx(\delta)\sqrt{\frac{\regfactor_\iteridx}{\regfactor_0}} + \sqrt{\frac{2\sigma_\loss^2}{\lossfactor\regfactor_0}\log(\det(\eye + \lossfactor\regfactor_0^{-1}\mat\kernel_\iteridx)^\half/\delta)}, \quad \forall\iteridx \in \N,
        \end{split}
    \end{equation}
    which is non-decreasing, so that:
    \begin{equation*}
        \beta_\iteridx^\star(\delta)
        \in \bigo(\sqrt{\regfactor_\iteridx}(\log \iteridx)^\ptailexp + \sqrt{\card{\domain}\log\iteridx})
        = \bigo((\log \iteridx)^{\ptailexp + \regfexp/2} + \sqrt{\card{\domain}\log\iteridx}),
    \end{equation*}
    as discussed in \autoref{rem:beta-log-t}. On $\event_G(\delta) \cap \initevent(\delta)$, the cumulative regret is then bounded by:
    \begin{equation}
        \begin{split}
            \Regret_\niter
            = \sum_{\iteridx=1}^\niter \regret_\iteridx
            &\leq 2\sum_{\iteridx=1} \beta_{\iteridx-1}(\delta) (\sigma_{\iteridx-1}(\location^\star) + \sigma_{\iteridx-1}(\location_\iteridx))\\
            &\leq 2\beta_\niter^\star(\delta) \sum_{\iteridx=1}  (\sigma_{\iteridx-1}(\location^\star) + \sigma_{\iteridx-1}(\location_\iteridx))\\
            &\leq
            2\beta_\niter^\star(\delta)
            \sqrt{
                2\niter
                \paren*{
                    \sum_{\iteridx=1}^\niter \sigma_{\iteridx-1}^2(\location^\star)
                    +
                    \sum_{\iteridx=1}^\niter \sigma_{\iteridx-1}^2(\location_\iteridx)
                }
            },
        \end{split}
    \end{equation}
    where an application of the Cauchy-Schwarz inequality yields the last line.
    Considering the sum of $\sigma_{\iteridx-1}^2(\location^\star)$, for large $\niter$, we have that $\sum_{\iteridx=1}^\niter \frac{1}{\iteridx}$ is $\bigo(\log\niter)$, so that, $\sum_{\iteridx=1}^\niter \frac{\log^\anyscalar\iteridx}{\iteridx}$ is $\bigo(\log^{\anyscalar+1}\niter)$, for any $\anyscalar > 0$. Now applying \autoref{thr:variance-decay} with $\regfactor_\iteridx \asymp \log^\regfexp\iteridx$, we obtain:
    \begin{equation}
        \sum_{\iteridx=1}^\niter \sigma_{\iteridx-1}^2(\location^\star)
        \in \bigo\paren*{
            (\log\niter)^{1 + \regfexp\deff/2}
        }.
        \label{eq:optim-var-sum}
    \end{equation}

    Regarding the sum of predictive variances at the selected points, well known information-gain bounds \citep{Srinivas2010, Chowdhury2017, Durand2018} yield:
    \begin{equation}
        \sum_{\iteridx=1}^\niter \sigma_{\iteridx-1}^2(\location_\iteridx)
        \leq \anyconstant_\sigma \log\det(\eye + \lossfactor\regfactor_0^{-1}\Kernel_\niter)
        \in \bigo(\card{\domain}\log\niter),
    \end{equation}
    following \autoref{rem:beta-log-t}, where $c_\sigma>0$ is a constant depending only on $\regfactor_0$, $\lossfactor$, and $\sup_{\location\in\domain}\kernel(\location,\location)$. Since $\card{\domain}$ is finite and $\regfexp\deff>0$, this term is asymptotically dominated by \eqref{eq:optim-var-sum}. Furthermore, as $\regfexp>1$ and $\ptailexp>0$, we have $\beta_\niter^\star(\delta)\in\bigo((\log\niter)^{\ptailexp+\regfexp/2})$. Combining these estimates with the cumulative regret decomposition above gives:
    \begin{equation}
        \begin{split}
            \Regret_\niter
            &\in \bigo\paren*{
                \sqrt{\niter}
                (\log\niter)^{\ptailexp+\regfexp/2}
                (\log\niter)^{(1+\regfexp\deff/2)/2}
            }\\
            &= \bigo\paren*{
                \sqrt{\niter}
                (\log\niter)^{\ptailexp+(\regfexp+1)/2+\regfexp\deff/4}
            }\\
            &= \bigo\paren*{
                \sqrt{\niter}
                (\log\niter)^{\ptailexp+\regfexp(1+\deff/4)}
            },
        \end{split}
    \end{equation}
    where the last line is a relaxation using $\regfexp>1$. The conclusion holds with probability at least $1-\delta$ after replacing $\delta$ by $\delta/2$ in the conditioning events, completing the proof.
\end{proof}

\section{Additional experiments}
\label{sec:ablation}

\begin{figure}[t]
    \centering
    \includegraphics[width=0.45\linewidth]{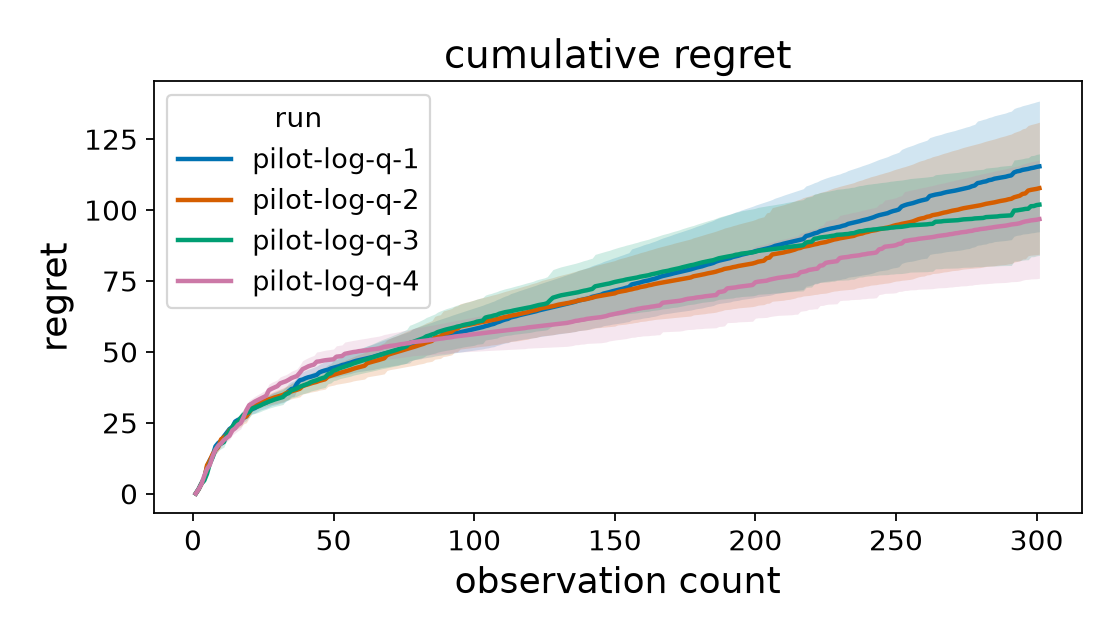}
    \includegraphics[width=0.45\linewidth]{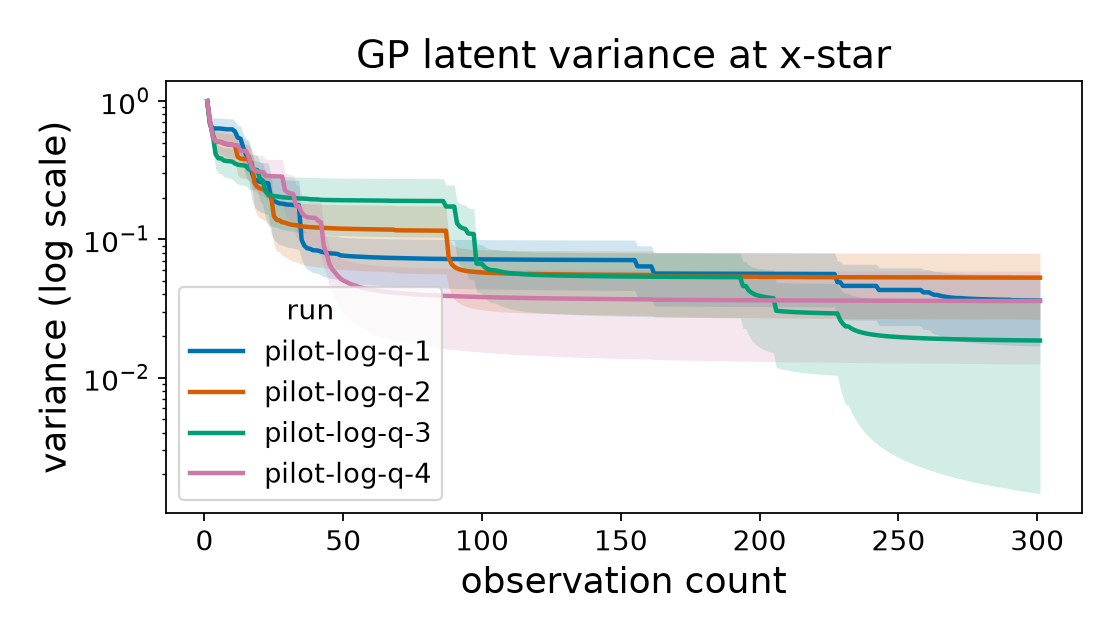}
    \caption{Effect of regularization schedule by varying exponent $q$ in $\regfactor_\iteridx \asymp \log^q \iteridx$. Results were averaged over 10 independent trials, following the same setup as the empirical illustrations in the main paper.}
    \label{fig:schedule}
\end{figure}

\begin{figure}[t]
    \centering
    \includegraphics[width=0.45\linewidth]{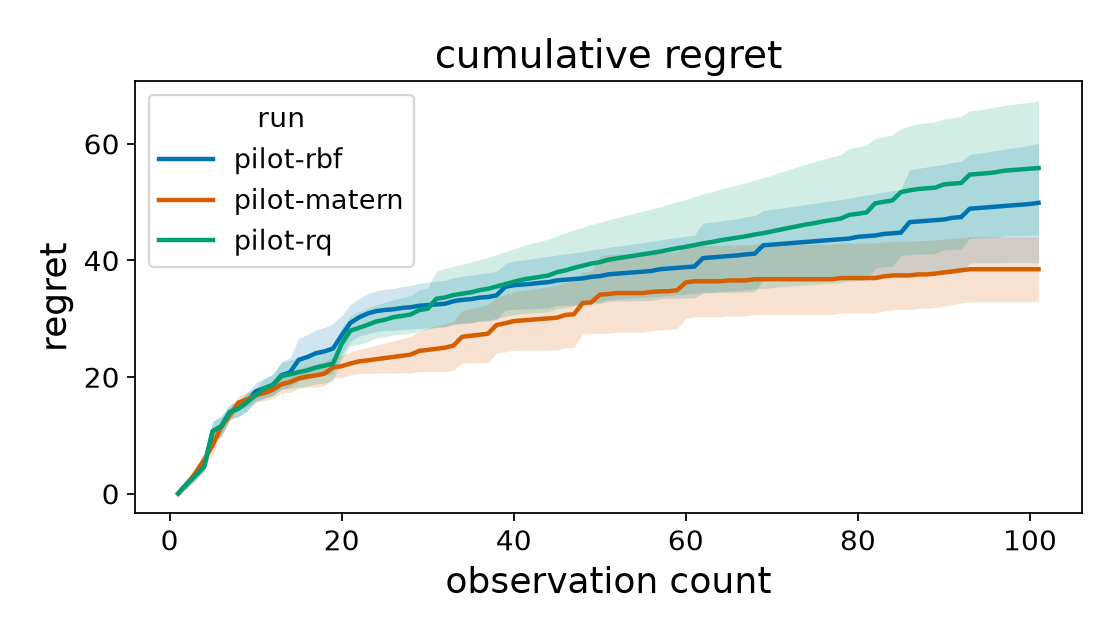}
    \includegraphics[width=0.45\linewidth]{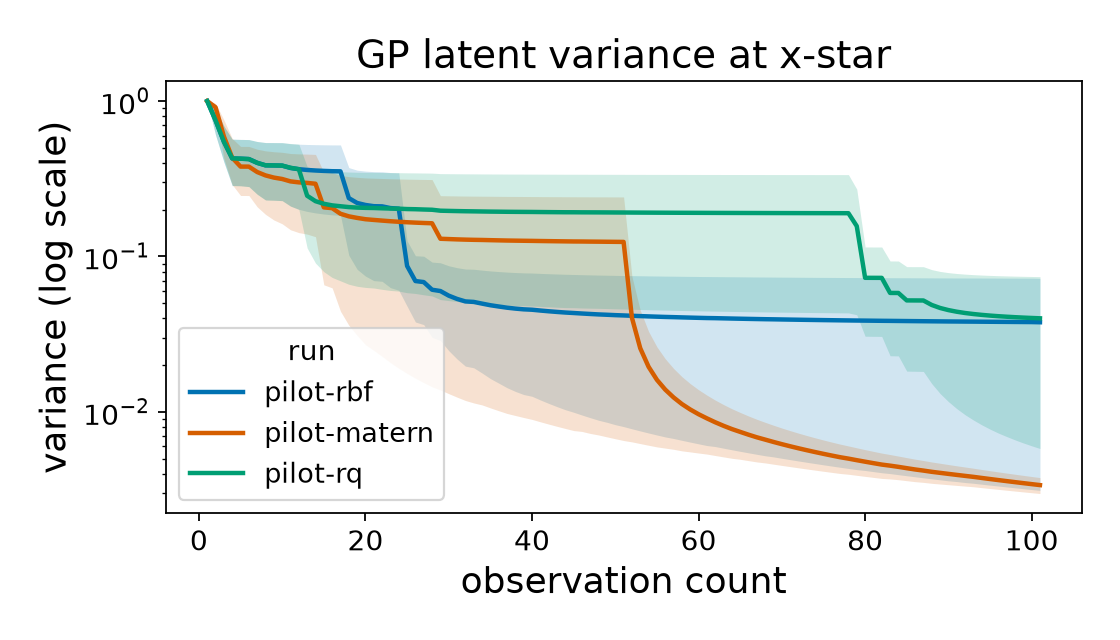}
    \caption{Effect of regularization kernel choice. Results were averaged over 10 independent trials, following the same setup as the empirical illustrations in the main paper.}
    \label{fig:kernel}
\end{figure}

\begin{figure}[t]
    \centering
    \includegraphics[width=0.32\linewidth]{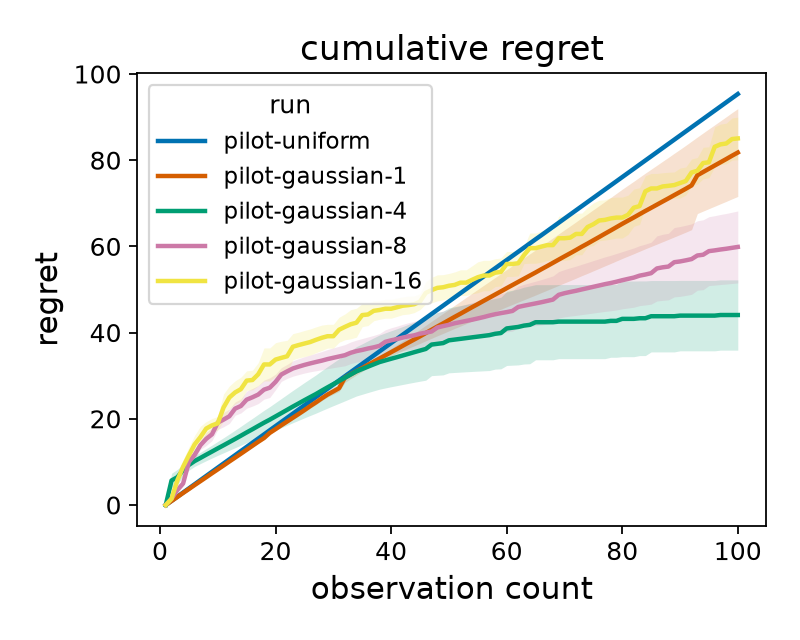}
    \includegraphics[width=0.32\linewidth]{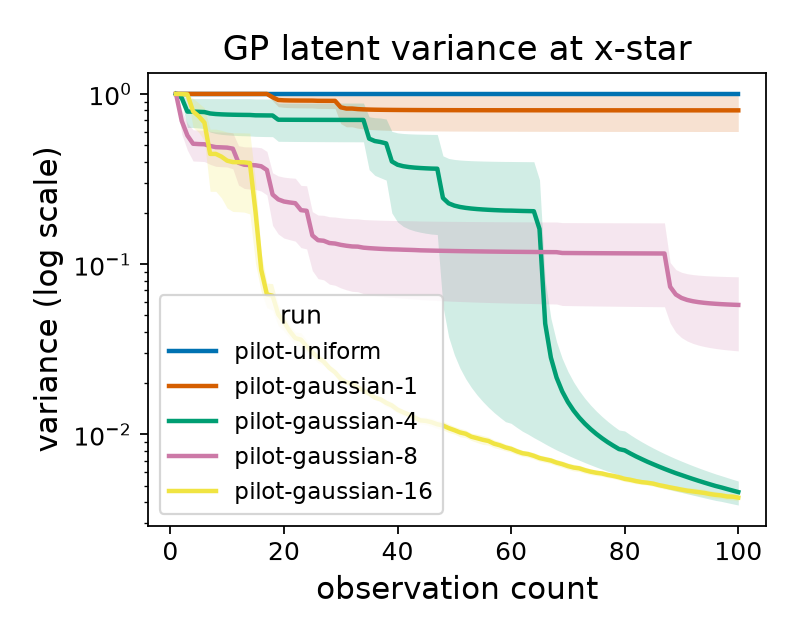}
    \includegraphics[width=0.32\linewidth]{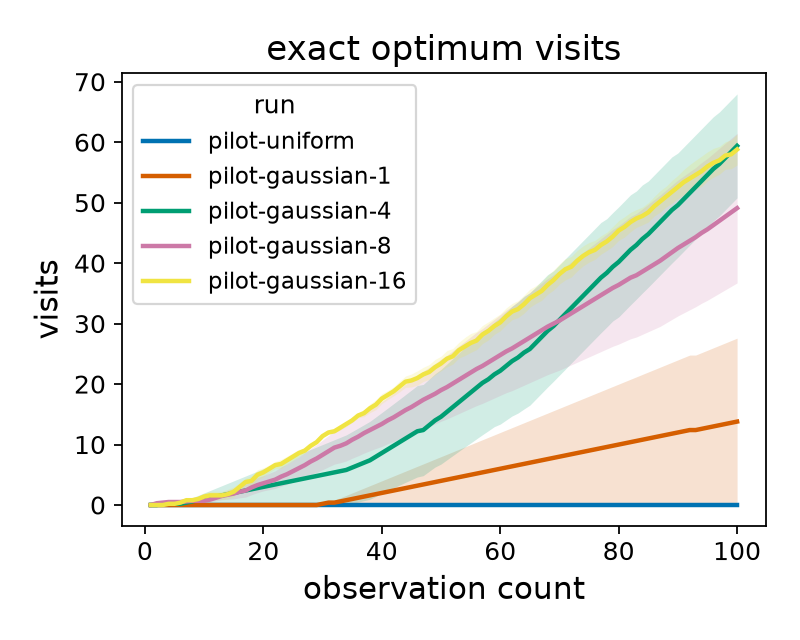}
    \caption{Effect of initialization prior on cumulative regret, auxiliary GP predictive variance at $\location^\star$, and exact visits to $\location^\star$, respectively. Results were averaged over 5 independent trials with a fixed objective and initial design.}
    \label{fig:prior}
\end{figure}

\paragraph{Sensitivity to the regularization schedule.} We additionally investigate the effect of the regularization exponent $q$, using $\lambda_t=0.25[1+\log(1+t)]^q$, for $q\in\{1,2,3,4\}$, while keeping the remaining experimental settings unchanged. Over the 300 evaluation rounds, \autoref{fig:schedule} shows that larger values of $q$ generally yield lower mean cumulative regret, although the differences vary over time and the variability across runs is substantial. Interestingly, this trend contrasts with the increasing polylogarithmic exponent in the upper bound of \autoref{thr:ts-regret}, illustrating that the theoretical rate does not necessarily reflect finite-budget performance. The case $q=1$ is included for comparison but falls outside the theorem's requirement $q>1$.

\paragraph{Effect of regularization kernel.} We additionally assess the effect of the parameter-space regularization kernel in \autoref{fig:kernel}, comparing the squared-exponential (RBF), Mat\'ern ($\nu=2.5$), and rational quadratic kernels \citep[Sec. 4.2.1]{Rasmussen2006} under the same experimental setup. The choice of kernel noticeably affects both optimization performance and the decay of the auxiliary GP predictive variance at the optimum. In this experiment, the Mat\'ern kernel yields the lowest cumulative regret and final predictive variance, while the rational quadratic kernel performs worst. This illustrates that the geometry induced by the parameter-space kernel can materially influence the practical behavior of the method.

\paragraph{Sensitivity to the initialization prior.} We investigate the effect of the initialization prior in \autoref{fig:prior}, comparing Gaussian (LeCun) priors with scales $1$, $4$, $8$, and $16$, and a uniform (Xavier) prior. The objective is a fixed realization of a neural network initialized from a Gaussian prior with scale $8$, and the initial design is kept fixed across configurations.
From a Bayesian perspective, the initialization prior encodes beliefs about plausible objective functions while also influencing exploration through randomized regularized training. In our framework, the prior controls the probability of favorable initializations under \autoref{a:prior}, while regularization determines how strongly the trained model retains the influence of its initialization. The experiments show that more concentrated priors lead to fewer visits to the optimum, weaker variance reduction, and higher cumulative regret. However, the most dispersed prior overexplores and does not achieve the lowest regret, despite visiting the optimum more frequently. Interestingly, the Gaussian prior with scale $4$ outperforms the scale $8$ prior used to generate the objective, illustrating that matching the generating prior need not yield the best finite-budget optimization performance under approximate inference.

\paragraph{Limitations of the empirical illustrations.} These experiments are intended to illustrate the theoretical mechanisms rather than provide a systematic empirical comparison of optimization methods or initialization strategies. We consider a single selected realization of a one-dimensional objective with a unique global optimum and competing local optima, allowing us to examine different exploration behaviors, including premature concentration in suboptimal regions. However, we do not investigate variability across different objective realizations, neural network architectures, or higher-dimensional domains. The observed trends should therefore be interpreted as problem-specific illustrations rather than general performance guarantees.


\newpage
\input{checklist.tex}

\end{document}

%% file: notation.tex
\newcommand*{\argmax}{\operatornamewithlimits{argmax}}
\newcommand*{\argmin}{\operatornamewithlimits{argmin}}

\DeclarePairedDelimiter{\norm}{\lVert}{\rVert}
\DeclarePairedDelimiter{\abs}{\lvert}{\rvert}
\DeclarePairedDelimiter{\inner}{\langle}{\rangle}
\DeclarePairedDelimiter{\paren}{(}{)}                       
\DeclarePairedDelimiter{\braces}{\lbrace}{\rbrace}          
\let\card\abs

\newcommand*{\placeholder}{\makebox[1.1ex]{$\cdot$}}
\newcommand*{\pch}{\placeholder}

\newcommand*{\field}[1]{\mathbb{\MakeUppercase{#1}}}                
\newcommand*{\set}[1]{{\mathcal{\MakeUppercase{#1}}}}			    
\newcommand*{\collection}[1]{{\mathfrak{\MakeUppercase{#1}}}}       

\newcommand*{\seqbracks}{\paren}
\newcommand*{\range}[1]{\braces{1, \dotsc, #1}}                     

\newcommand*{\functional}[1]{{\MakeUppercase{#1}}}

\newcommand*{\operator}[1]{{\MakeUppercase{#1}}}

\newcommand*{\R}{\field{R}} 
\newcommand*{\N}{\field{N}} 

\renewcommand{\vec}[1]{{\boldsymbol{\mathbf{#1}}}}
\newcommand*{\mat}[1]{\vec{\MakeUppercase{#1}}}
\newcommand*{\eye}{\mat{I}}							
\newcommand*{\transpose}{\mathsf{T}}
\newcommand*{\tr}{{\operatorname{tr}}}						

\newcommand*{\idop}{\operator{I}}

\newcommand*{\dataset}{\set{D}} 				
\newcommand*{\observation}{y} 					

\newcommand*{\obsnoise}{\epsilon}

\newcommand*{\gp}{\mathcal{GP}}

\newcommand*{\kernel}{k}

\newcommand*{\Kernel}{\mat{\kernel}}
\newcommand*{\keigval}{\kappa}

\newcommand*{\expectation}{\mathbb{E}}

\newcommand*{\variance}{\mathbb{V}}

\newcommand*{\prob}[1]{\mathbb{P}\left[ #1 \right]}

\newcommand*{\indicator}{\mathbb{I}} 				

\newcommand*{\filtration}{\collection{F}}
\newcommand*{\event}{\set{E}}

\newcommand*{\Hspace}{\set{H}} 						
\newcommand*{\feature}{\phi}

\newcommand*{\Hessian}{\operator{H}}

\newcommand*{\af}{a} 						
\newcommand*{\regret}{r}					
\newcommand*{\Regret}{\functional{R}}					

\newcommand*{\objective}{f}

\newcommand*{\slocation}{x} 
\newcommand*{\location}{\vec{\slocation}} 
\newcommand*{\domain}{\set{X}} 

\newcommand*{\dimension}{d}
\newcommand*{\locdim}{\dimension}

\newcommand*{\model}{g}
\newcommand*{\mdlspace}{\set{G}}
\newcommand*{\parameter}{\theta}
\newcommand*{\parameters}{{\parameter}}
\newcommand*{\paramdim}{M}
\newcommand*{\paramspace}{\Theta} 

\newcommand*{\weight}{w}    

\newcommand*{\paramprior}{\Pi_0}

\newcommand*{\deff}{{d_{\mathrm{eff}}}}

\newcommand*{\lossfactor}{\alpha_\loss}

\newcommand*{\error}{\varepsilon}

\newcommand*{\vol}[1]{\operatorname{vol}(#1)}

\newcommand*{\mlabel}{z}                        

\newcommand*{\vnoise}{\xi}
\newcommand*{\vncov}{\Sigma}

\newcommand*{\priorradius}{\alpha}
\newcommand*{\pradius}{\priorradius}
\newcommand*{\pfunction}{\psi_0}
\newcommand*{\ptailexp}{\zeta}
\newcommand*{\distbound}{\bound}
\newcommand*{\initevent}{\event_{\mathrm{init}}}

\newcommand*{\paramdist}{\distance}
\newcommand*{\regfexp}{q}

\newcommand*{\loss}{\ell}
\newcommand*{\Loss}{L}
\newcommand*{\regfun}{R}
\newcommand*{\regfactor}{\lambda}
\newcommand*{\obsspace}{\set{Y}}

\newcommand{\utility}{u}

\newcommand*{\clabel}{z}    
\let\threshold\thresh

\newcommand*{\iteridx}{t}

\newcommand*{\niter}{T}
\newcommand*{\nobs}{n}

\newcommand*{\nfeatures}{m}

\newcommand*{\bigo}{\set{O}}        

\newcommand*{\distance}{\rho} 

\newcommand*{\bound}{b}		

\newcommand*{\anyconstant}{c}

\newcommand*{\anyscalar}{s}

\newcommand*{\anyfunction}{h}
\newcommand*{\anyfunctions}{\operator{F}}

\newcommand*{\anyoperator}{\operator{A}}

\newcommand*{\half}{{\nicefrac{1}{2}}}

\newcommand*{\iid}{i.i.d.\xspace}

%% file: theoremenvs.tex
\declaretheorem{assumption}

\declaretheorem[numberwithin=section]{theorem}
\declaretheorem[numberlike=theorem]{lemma, corollary, proposition}

\declaretheorem[style=remark]{remark}
\declaretheorem{definition}

%% file: checklist.tex
\section*{NeurIPS Paper Checklist}

\begin{enumerate}

\item {\bf Claims}
    \item[] Question: Do the main claims made in the abstract and introduction accurately reflect the paper's contributions and scope?
    \item[] Answer: \answerYes{}
    \item[] Justification: The abstract and introduction describe the theoretical framework and the scope of the regret guarantee. Sections~3--4 state the formal results, Section~5 provides a numerical illustration, and Section~7 discusses limitations.

\item {\bf Limitations}
    \item[] Question: Does the paper discuss the limitations of the work performed by the authors?
    \item[] Answer: \answerYes{}
    \item[] Justification: Section~7 discusses the finite-domain, realizability, and unique-optimum assumptions, as well as the additional arguments needed for general LFBO acquisition functions and continuous domains.

\item {\bf Theory assumptions and proofs}
    \item[] Question: For each theoretical result, does the paper provide the full set of assumptions and a complete (and correct) proof?
    \item[] Answer: \answerYes{}
    \item[] Justification: The paper states the assumptions used by the theoretical results in Sections~3--4 and provides proofs in Appendix~B. The main results include Lemma~3.2, Theorem~4.1, Corollary~4.2, Lemma~4.3, and Theorem~4.4.

\item {\bf Experimental result reproducibility}
    \item[] Question: Does the paper fully disclose all the information needed to reproduce the main experimental results of the paper to the extent that it affects the main claims and/or conclusions of the paper (regardless of whether the code and data are provided or not)?
    \item[] Answer: \answerYes{}
    \item[] Justification: Section~5 describes the experimental design. The linked repository provides the configurations, seeds, and commands needed to reproduce Figure~1.

\item {\bf Open access to data and code}
    \item[] Question: Does the paper provide open access to the data and code, with sufficient instructions to faithfully reproduce the main experimental results, as described in supplemental material?
    \item[] Answer: \answerYes{}
    \item[] Justification: The public repository linked in Section~5 will provide the code and instructions to generate the synthetic observations, run the experiment, and reproduce Figure~1.

\item {\bf Experimental setting/details}
    \item[] Question: Does the paper specify all the training and test details (e.g., data splits, hyperparameters, how they were chosen, type of optimizer) necessary to understand the results?
    \item[] Answer: \answerYes{}
    \item[] Justification: Section~5 describes the grid, teacher architecture, observation noise, initialization, optimizer and training budget, regularization schedule, and diagnostic GP settings.

\item {\bf Experiment statistical significance}
    \item[] Question: Does the paper report error bars suitably and correctly defined or other appropriate information about the statistical significance of the experiments?
    \item[] Answer: \answerYes{}
    \item[] Justification: Figure~1 shows the individual trajectories and pointwise means across ten independent policy seeds. The individual trajectories display variability across runs.

\item {\bf Experiments compute resources}
    \item[] Question: For each experiment, does the paper provide sufficient information on the computer resources (type of compute workers, memory, time of execution) needed to reproduce the experiments?
    \item[] Answer: \answerYes{}
    \item[] Justification: Section~5 reports the number of runs, sequential rounds, and optimization steps per fit. The accompanying repository will document the compute resources.

\item {\bf Code of ethics}
    \item[] Question: Does the research conducted in the paper conform, in every respect, with the NeurIPS Code of Ethics \url{https://neurips.cc/public/EthicsGuidelines}?
    \item[] Answer: \answerYes{}
    \item[] Justification: The work is theoretical and concerns mathematical guarantees for adaptive optimization algorithms. It does not involve human subjects, private data, released models, or other components that require special ethical safeguards beyond standard research integrity.

\item {\bf Broader impacts}
    \item[] Question: Does the paper discuss both potential positive societal impacts and negative societal impacts of the work performed?
    \item[] Answer: \answerYes{}
    \item[] Justification: The work is foundational and aims to improve principled adaptive optimization methods. Potential positive impacts include improved data efficiency in scientific and engineering applications; possible negative impacts are indirect and depend on downstream uses of black-box optimization methods.

\item {\bf Safeguards}
    \item[] Question: Does the paper describe safeguards that have been put in place for responsible release of data or models that have a high risk for misuse (e.g., pre-trained language models, image generators, or scraped datasets)?
    \item[] Answer: \answerNA{}
    \item[] Justification: The paper does not release models, datasets, or systems with high misuse risk.

\item {\bf Licenses for existing assets}
    \item[] Question: Are the creators or original owners of assets (e.g., code, data, models), used in the paper, properly credited and are the license and terms of use explicitly mentioned and properly respected?
    \item[] Answer: \answerNA{}
    \item[] Justification: The paper does not use or redistribute datasets, code, or other external assets beyond citing prior academic work.

\item {\bf New assets}
    \item[] Question: Are new assets introduced in the paper well documented and is the documentation provided alongside the assets?
    \item[] Answer: \answerYes{}
    \item[] Justification: The codebase for the empirical illustrations will be released on GitHub.

\item {\bf Crowdsourcing and research with human subjects}
    \item[] Question: For crowdsourcing experiments and research with human subjects, does the paper include the full text of instructions given to participants and screenshots, if applicable, as well as details about compensation (if any)?
    \item[] Answer: \answerNA{}
    \item[] Justification: The paper does not involve crowdsourcing or human-subject research.

\item {\bf Institutional review board (IRB) approvals or equivalent for research with human subjects}
    \item[] Question: Does the paper describe potential risks incurred by study participants, whether such risks were disclosed to the subjects, and whether Institutional Review Board (IRB) approvals (or an equivalent approval/review based on the requirements of your country or institution) were obtained?
    \item[] Answer: \answerNA{}
    \item[] Justification: The paper does not involve human-subject research.

\item {\bf Declaration of LLM usage}
    \item[] Question: Does the paper describe the usage of LLMs if it is an important, original, or non-standard component of the core methods in this research? Note that if the LLM is used only for writing, editing, or formatting purposes and does \emph{not} impact the core methodology, scientific rigor, or originality of the research, declaration is not required.
    \item[] Answer: \answerNA{}
    \item[] Justification: The core method development does not involve LLMs as a methodological component.

\end{enumerate}